\documentclass[11pt,a4paper,logo,copyright,nonumbering]{map}
\usepackage{times,inconsolata}
\usepackage{parskip}
\usepackage[T1]{fontenc}
\usepackage[utf8]{inputenc}
\usepackage{amsmath,amssymb,amsthm}
\usepackage{listings}
\usepackage{booktabs}
\usepackage{graphicx,booktabs,array,tabularx,multirow}
\usepackage{xcolor,microtype,url}
\usepackage{subcaption}
\usepackage{algorithm,algpseudocode}
\usepackage{tikz,pgfplots}
\usepackage[most]{tcolorbox}
\usepackage{makecell}
\usepackage{booktabs}
\usepackage{adjustbox}
\newtcolorbox{casebox}[1]{
    enhanced,
    breakable,
    colback=white,
    colframe=black,
    boxrule=0.8pt,
    arc=3pt,
    title=\textbf{#1},
    fonttitle=\bfseries,
    before skip=10pt,
    after skip=12pt
}
\usetikzlibrary{arrows.meta,positioning,calc}
\pgfplotsset{compat=1.18}
\usepackage{placeins,needspace}
\usepackage{hyperref}
\usepackage{ragged2e}
\hypersetup{colorlinks=true,linkcolor=blue,citecolor=purple,urlcolor=cyan,
  pdftitle={FORT-Searcher: Synthesizing Shortcut-Resistant Search Tasks for Training Deep Search Agents},
  pdfauthor={Jia Deng, Yimeng Chen, Xiaoqing Xiang, Ziyang Zeng, Shuo Tang, Zhengbo Zhang, Wayne Xin Zhao, Feng Chang, Chuan Hao, Yuan Wei, Ran Tao, Bryan Dai, Ji-Rong Wen},pdfsubject={FORT-Searcher technical report}}
\newcolumntype{Y}{>{\raggedright\arraybackslash}X}
\newcommand{\fort}{\textsc{FORT}}

\theoremstyle{plain}
\newtheorem{proposition}{Proposition}[section]
\newtheorem{corollary}[proposition]{Corollary}
\theoremstyle{definition}

\title{%
\vspace{-0.2in}
\centering\fontsize{15pt}{16pt}\selectfont
FORT-Searcher: Synthesizing Shortcut-Resistant Search Tasks \\
for Training Deep Search Agents
\vspace{-0.2in}}
\author{
Jia Deng\textsuperscript{1*},
Yimeng Chen\textsuperscript{2*},
Xiaoqing Xiang\textsuperscript{1*},
Ziyang Zeng\textsuperscript{3*},
Shuo Tang\textsuperscript{4*},
Zhengbo Zhang\textsuperscript{5*},\\
\normalfont
Wayne Xin Zhao\textsuperscript{1$\dagger$},
\normalfont
Feng Chang\textsuperscript{3$\dagger$$\ddagger$},
Chuan Hao\textsuperscript{3},
Yuan Wei\textsuperscript{3},\\
\normalfont
Ran Tao\textsuperscript{3},
Bryan Dai\textsuperscript{3},
Ji-Rong Wen\textsuperscript{1}\\
\normalfont
\textsuperscript{1}Gaoling School of Artificial Intelligence, Renmin University of China, \\
\normalfont
\textsuperscript{2}KAUST,
\normalfont
\textsuperscript{3}IQuest Research,
\textsuperscript{4}Shanghai Jiao Tong University \\
\normalfont
\textsuperscript{5}Institute of Automation, Chinese Academy of Sciences \\
\normalfont
\normalsize{
\textsuperscript{*}Equal contributors \quad
\textsuperscript{$\dagger$}Corresponding Authors \quad
\textsuperscript{$\ddagger$}Project Leader \\
\normalfont
Email: \texttt{dengjia0510@outlook.com},
\texttt{yimeng.chen@kaust.edu.sa},\\
\normalfont
\texttt{batmanfly@gmail.com},
\texttt{fchang@iquestlab.com}
}
}
\begin{abstract}
Training deep-search agents requires verifiable questions that demand
substantial evidence acquisition, yet structurally complex tasks may still
admit cheap identifying routes when clues are overly selective, co-covered
by the same evidence, directly searchable, or resolved from solver priors.
We formalize this gap with a shortcut-aware difficulty framework that
distinguishes intended construction complexity from identifying routes
available through retrieval. Guided by this framework, we introduce
\textsc{FORT}, a \textbf{F}ramework \textbf{o}f Shortcut-\textbf{R}esistant \textbf{T}raining-Data Synthesis,
which targets four shortcut risks across entity selection, evidence graph
construction, question formulation, and adversarial refinement. We evaluate
realized search behavior with solver-conditioned trajectory diagnostics rather
than graph complexity or search length alone. Compared with existing
open-source deep-search datasets, \textsc{FORT} shows later answer exposure,
higher retrieval effort, lower single-clue selectivity, greater evidence
dispersion, and higher realized pre-answer dependency costs. Using the
resulting trajectories, we train FORT-Searcher with supervised fine-tuning
alone. Across model scales, \textsc{FORT}-Searcher consistently improves over its backbones, achieving the strongest overall performance among
similarly sized agents.
\end{abstract}

\begin{document}
\maketitle
\section{Introduction}
\label{sec-intro}
Large language models (LLMs) increasingly serve as tool-empowered search agents~\citep{chen2025cpo,hu2026agentfugueagentscalinglonghorizon}. Effective open-web search requires more than retrieval: agents must iteratively refine queries, locate hard-to-find information, assess evidence reliability, and connect scattered clues~\citep{chen2025iterresearch,qiao2025webresearcher,hu2026samstateadaptivememorylonghorizon,zeng2026optimizing}. Training these agents with supervised fine-tuning (SFT) or reinforcement learning requires verifiable questions that demand search~\citep{song2025smartsearcher,song2025r1searcher,mei2026searchart,thakur2026orbit}. Yet traditional retrieval datasets often do not require long-horizon evidence acquisition~\citep{yang2018hotpotqa}.


To address this limitation, recent deep search task synthesis methods often increase apparent difficulty by enriching the intended solving structure, including hop count and graph shape~\citep{DBLP:conf/emnlp/SunSWRJZBDZLFWW25,trivedi2022musique,tao2025webshaper,team2026mirothinker}, hierarchical constraints~\citep{xia2025open}, and evidence dispersion or treewidth~\citep{chu2026redsearcher}. However, the underlying structural priors used to generate the question-to-answer chain do not necessarily translate into the realized search process when the task is actually executed by the search agent. 
The agent may still bypass much of the intended evidence acquisition by leveraging evidence from the environment or its own prior knowledge. Such shortcuts may limit the training signal for long-horizon search, planning, and information synthesis.

To characterize this mismatch, we develop a shortcut-aware difficulty
framework for multi-constraint agentic retrieval. It links difficulty to
intended task complexity and how the question is realized by the retrieval
environment and solver. Many clues and a long evidence graph need not
prevent a cheaper identifying route. We call these route-level collapses,
together with solver-level prior binding, \emph{shortcuts}. Data synthesis
should therefore assess whether questions with complex structures admit
cheap bypasses in real search trajectories.

This framework leads to four actionable shortcut risks that should be controlled during data synthesis: \textit{evidence co-coverage}, where multiple clues
can be verified from the same evidence item; \textit{single-clue selectivity},
where one clue or a small clue subset is already enough to identify the answer;
\textit{exposed constants}, where names, strings, or numbers on the question
surface make downstream queries directly executable; and \textit{prior-knowledge
binding}, where the solver names the answer before retrieval anchors it. Since
these risks are only partially visible from the final question alone, we monitor
their realized effects through trajectory signatures, including solving cost, answer hit time, and answer-before-evidence rate.

Guided by this framework, we introduce \textsc{FORT}, a \textbf{F}ramework \textbf{o}f Shortcut-\textbf{R}esistant \textbf{T}raining-Data Synthesis.  \textsc{FORT} targets
shortcut risks across the full data construction process. It selects long-tail
root entities to reduce prior-knowledge binding, initializes and expands
evidence graphs to avoid linearly exposed structures, collects facts from
heterogeneous sources to reduce evidence co-coverage, constructs derived facts
that are less likely to appear verbatim in a single source, and selects facts
that are weak in isolation but identifying in combination. During question
formulation, it withholds intermediate entity names and fuzzes exact constants
into truthful ranges, categories, or indirect descriptions. Finally,
adversarial refinement runs a strong search agent against each draft question
and repairs shortcut-prone, ambiguous, or over-obfuscated cases.

Compared with existing open-source deep-search datasets, \textsc{FORT} exhibits later answer exposure and stronger shortcut-related factor diagnostics under the evaluated search settings. Using these data, we train FORT-Searcher with supervised fine-tuning alone,
with its 35B variant achieving the strongest overall score among similarly
sized agents.

    

Our contributions are:
\begin{itemize}
    \item We develop a framework for multi-constraint agentic retrieval that relates realized difficulty to the cheapest identifying route and identifies four actionable shortcut risks.

    \item We propose \textsc{FORT}, which targets shortcut risks across entity selection, evidence graph construction, question formulation, and adversarial refinement.
    \item Experiments show stronger shortcut-related diagnostics and improved downstream performance from FORT training data.

\end{itemize}
\section{Search Tasks and Shortcut Risks}
\label{sec:shortcuts}
\subsection{Task and scope}
\label{sec:task}

We study \emph{multi-constraint agentic retrieval}: an agent iteratively
queries a retrieval interface (e.g., a web search engine), examines evidence,
and decides when to answer. We use the shorthand
\begin{equation}
q=(\mathcal X,\mathcal C_q,\Sigma),
\label{eq:problem_tuple}
\end{equation}
where $\mathcal X$ is the task-relevant space of candidate entities or
determinate values, $\mathcal C_q$ is the set of factual constraints expressed
by the question, and $\Sigma$ is the retrieval interface~\citep{wang2026search}. Each constraint
$c:\mathcal X\rightarrow\{0,1\}$ indicates whether a candidate satisfies it.
The question wording is implicit in $q$ and determines initially
available query constants.
For any subset $\mathcal P\subseteq\mathcal C_q$, define
\begin{equation}
\mathrm{Ans}(\mathcal P)
=
\{x\in\mathcal X:\ c(x)=1 \text{ for all }c\in\mathcal P\}.
\label{eq:answer}
\end{equation}
For the formal analysis, we assume
$\mathrm{Ans}(\mathcal C_q)=\{y^\star\}$, so the full constraint set
uniquely identifies the target. The solver observes the question and has
access to $\Sigma$, but not the construction graph or supporting evidence.
A subset $\mathcal P\subseteq\mathcal C_q$ is \emph{identifying} if
$\mathrm{Ans}(\mathcal P)=\{y^\star\}$. Hence, a complex graph may still
admit a short solution route if an identifying subset is cheaply retrievable.

\subsection{Evidence-based difficulty and identifying routes}
\label{sec:difficulty}

A complex construction graph does not by itself guarantee costly
evidence acquisition. To characterize the search a task requires, we
consider an \emph{evidence-only reference solver} that issues only
executable queries, using constants available from the question or
previously retrieved evidence; returns an answer only after retrieved
evidence verifies an identifying subset; and uses no question-specific
prior knowledge. This idealized class defines a structural reference cost, rather
than an implementable open-web policy. Let $D_{\mathrm{post}}(q)$ (the pure-posterior cost of
Appendix~\ref{app:formal}) denote the infimum expected
number of retrieval calls, conditional on success, over reference
policies with nonzero success probability.
This success-conditioned cost permits rare successful runs and need
not reflect the effort of reliable route discovery.
Each retrieval invocation
is counted separately.

Let $\mathcal I_q$ denote the family of identifying subsets defined by
$\mathrm{Ans}(\mathcal P)=\{y^\star\}$. For
$\mathcal P\in\mathcal I_q$, a \emph{valid identifying route} is a
sequence of executable queries whose accumulated evidence verifies
$\mathcal P$ for $y^\star$. Write $Q_\Sigma(\mathcal P)$ for its minimum
number of retrieval calls. Two quantities further constrain this cost:
$M_{\mathrm{ev}}(\mathcal P)$ is the minimum number of calls needed to
obtain evidence verifying $\mathcal P$ when query executability is
ignored; $\mathrm{dep}(\mathcal P)$ is the minimum dependency depth
over valid routes, where a route's depth is the number of queries in
its longest chain linked by required constants obtained from earlier
retrievals. Formal definitions are given in Appendix~\ref{app:formal}.

\begin{proposition}[Evidence-acquisition lower bound]
\label{prop:difficulty}
Under the reference-solver restrictions above,
\begin{equation}
\begin{aligned}
D_{\mathrm{post}}(q)
&\ge Q_\Sigma^\star
:=\min_{\mathcal P\in\mathcal I_q}Q_\Sigma(\mathcal P)\\
&\ge\min_{\mathcal P\in\mathcal I_q}
\max\{M_{\mathrm{ev}}(\mathcal P),\mathrm{dep}(\mathcal P)\}.
\end{aligned}
\label{eq:difficulty_bound}
\end{equation}
If an eligible reference policy can identify and follow a cheapest
valid route without exploratory retrieval, then
$D_{\mathrm{post}}(q)=Q_\Sigma^\star$.
\end{proposition}

Every successful reference trajectory is a valid route for an identifying
subset, costing at least $Q_\Sigma^\star$. Every valid route also incurs at
least the unconstrained evidence-acquisition cost and contains its dependency
chains, yielding both inequalities; Appendix~\ref{app:formal} gives full proofs.

The minimum covers all identifying subsets and valid routes, including
unintended ones. If verification requires at least one retrieval, an initially
executable query verifying an identifying subset gives $Q_\Sigma^\star=1$,
regardless of the construction graph's size or depth
(Corollary~\ref{cor:app_one_query}). This establishes cheap-route existence;
a particular solver may need additional search to discover it. Conversely,
a large $Q_\Sigma^\star$ forces substantial retrieval cost on every successful
reference trajectory.


\subsection{Four shortcut risks}
\label{sec:risks}

Section~\ref{sec:difficulty} motivates four actionable shortcut risks.
\emph{Single-clue selectivity} can make small subsets identifying
(enlarging $\mathcal I_q$); \emph{evidence co-coverage} can reduce verification
calls (lowering $M_{\mathrm{ev}}$); and \emph{exposed constants} can make
downstream queries executable earlier (lowering $\mathrm{dep}$).
These mechanisms can enable cheaper identifying routes. \emph{Prior-knowledge binding} acts at the solver level: prior knowledge may
identify the answer before the solver acquires reference-solver evidence.

Risks can overlap: an exact quotation may be both an exposed constant and a
highly selective clue. Table~\ref{tab:risks} maps risks to construction
controls; Appendix~\ref{app:shortcut_analysis} gives analysis and diagnostic
examples.
\begin{table}[t]
\centering
\caption{Shortcut mechanisms and representative \textsc{FORT} controls. Details are shown in Appendix~\ref{app:shortcut_analysis}.}
\normalsize
\renewcommand{\arraystretch}{1.0}
\setlength{\tabcolsep}{3pt}

\begin{minipage}{\linewidth}
\begin{tabularx}{\linewidth}{@{}
>{\raggedright\arraybackslash}p{0.18\linewidth}
>{\raggedright\arraybackslash}X
>{\raggedright\arraybackslash}p{0.28\linewidth}
@{}}
\toprule
\textbf{Shortcut risk}
&
\textbf{Shortcut mechanism}
&
\textbf{\textsc{FORT} control} \\
\midrule

\textbf{Evidence co-coverage}
&
One result from $\Sigma$ supports several constraints in an identifying
subset $\mathcal P$, allowing joint evidence acquisition.
&
Multi-source enrichment; derived fact construction. \\

\midrule

\textbf{Single-clue selectivity}
&
For some $c\in\mathcal C_q$, $\mathrm{Ans}(\{c\})$ contains only
$y^\star$ or a few candidates.
&
Generic fact selection; low-specificity clue formulation. \\

\midrule

\textbf{Exposed constants}
&
Surface constants make downstream queries directly executable, bypassing
their discovery.
&
Cycle-based initialization; name withholding; exact-value fuzzing. \\

\midrule

\textbf{Prior-knowledge binding}
&
The solver links clues in $\mathcal C_q$ to $y^\star$ from prior knowledge
before retrieving supporting evidence.
&
Long-tail root selection; adversarial refinement. \\

\bottomrule
\end{tabularx}
\end{minipage}%

\label{tab:risks}
\end{table}

\subsection{Observable trajectory diagnostics}
\label{sec:diagnostics}

To examine how the shortcut risks in Section~\ref{sec:risks} manifest
during search, we define three observable diagnostics for a successful
trajectory $\tau$ under a fixed solver, retrieval interface $\Sigma$,
and search budget.
\paragraph{Solving cost.}
Solving cost counts retrieval invocations separately, including multiple
calls within one assistant interaction:
\begin{equation}
\Omega(\tau)=|\tau|.
\end{equation}
Large $\Omega(\tau)$ indicates substantial retrieval effort, but an answer
may appear early, followed by verification or detours. We therefore also
measure when the answer first appears.
\paragraph{Answer hit time.}
Let $T_{\mathrm{tool}}(\tau)$ and $T_{\mathrm{model}}(\tau)$ denote
the first assistant round in which the gold answer $y^*$ or a normalized
alias appears in retrieved evidence and visible model text, respectively,
with behavior before the first retrieval round assigned index $0$.
Absent occurrences have time and event-order position $+\infty$.
We define
\begin{equation}
T_{\mathrm{hit}}(\tau)
=
\min\{T_{\mathrm{tool}}(\tau),T_{\mathrm{model}}(\tau)\}.
\label{eq:answer_hit_time}
\end{equation}
Unlike $\Omega$, which counts retrieval calls, $T_{\mathrm{hit}}$ measures
earliest answer exposure in assistant rounds, not completed verification or
sufficient evidence. Evidence co-coverage, single-clue selectivity, and exposed
constants can expose answers earlier through cheaper identifying routes;
$I_{\mathrm{prior}}$ separately captures model-side answer-before-evidence
behavior. The positive correlation between $T_{\mathrm{hit}}$ and realized
pre-answer dependency cost (Appendix~\ref{app:data_comp}) indicates that later
exposure tends to accompany costlier dependencies. We therefore use
$T_{\mathrm{hit}}$ as an observable diagnostic, not a direct estimate of
minimum route cost.
\paragraph{Answer-before-evidence indicator.}
To capture observable answer-before-evidence behavior associated
with solver-side prior binding, we record whether the model names the gold answer before it first appears in retrieved evidence:
\begin{equation}
I_{\mathrm{prior}}(\tau)
=
\mathbf{1}
\left[
\operatorname{pos}_{\mathrm{model}}(y^*)
<
\operatorname{pos}_{\mathrm{tool}}(y^*)
\right],
\label{eq:prior_shortcut_indicator}
\end{equation}
where $\operatorname{pos}$ follows trajectory event order. This behavioral
proxy does not directly reveal internal prior knowledge; human annotation
validates its correspondence to shortcut-like answer-before-evidence
behavior (Appendix~\ref{app:prior_validation}). Successful trajectories that
name the answer before retrieving supporting evidence may cost less than
the no-prior reference solver.


For dataset-level comparisons, we report solver accuracy over all
evaluated questions and compute $\widehat{\Omega}$,
$T_{\mathrm{hit}}$, and $\widehat{p}_{\mathrm{prior}}$ over all
successful trajectories. Accuracy measures overall solvability, while
the trajectory diagnostics characterize search behavior conditional
on success.

To compare realized answer exposure across open-source deep-search datasets,
we evaluate all datasets under the same solver and retrieval budget.
Figure~\ref{fig:dataset_difficulty} shows that high retrieval cost can still
coexist with early answer exposure. For example, OpenSeeker has
$\widehat{\Omega}=50.2$, but $T_{\mathrm{hit}}=2.5$ and
$\widehat{p}_{\mathrm{prior}}=20.0\%$. This motivates shortcut-aware
controls beyond retrieval effort alone.
\FloatBarrier
\section{FORT: Shortcut-Resistant Task Synthesis}
\label{sec:method}
\fort{} synthesizes search tasks through four stages targeting the shortcut
risks in Section~\ref{sec:risks}: entity selection and graph initialization,
evidence graph construction, question formulation, and adversarial refinement,
as shown in Figure~\ref{fig:workflow}. The construction graph
$\mathcal G=(V,E)$ represents entities as nodes and verified relations
as edges, providing sourced evidence for question formulation
\citep{wu2025think,DBLP:conf/emnlp/SunSWRJZBDZLFWW25}. Adversarial search
then evaluates each draft using the diagnostics in
Section~\ref{sec:diagnostics}.


\begin{figure*}[t]
\centering
\includegraphics[width=\textwidth]{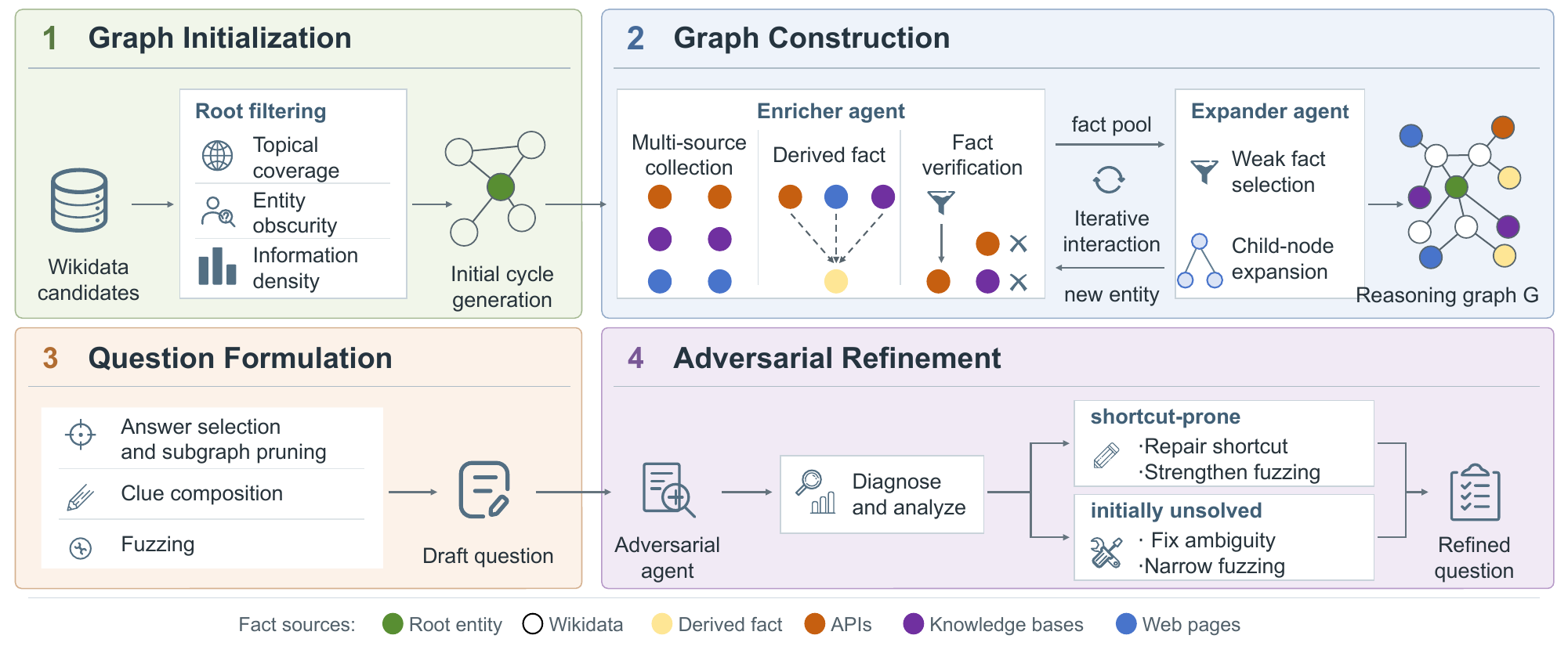}
\caption{Overview of \fort{}. A selected entity seeds an evidence graph,
which is expanded through fact enrichment and selection. A formulation
agent turns a selected subgraph into a question, and adversarial refinement
revises shortcuts or overly broad clues.}
\label{fig:workflow}
\end{figure*}

\subsection{Entity Selection and Graph Initialization}
\label{sec:initialization}

We select a root entity $r$ from broad Wikidata semantic categories,
excluding abstract concepts and preferring long-tail entities, especially
those without an English Wikipedia page. A lightweight pre-search removes
entities with insufficient external evidence.

When available, a pre-mined Wikidata cycle containing $r$ initializes
$\mathcal G_0$; otherwise, we use $(\{r\},\varnothing)$. Cycle
initialization supports indirect relational descriptions while hiding
intermediate entity names, reducing exposed-constant shortcuts.
Appendix~\ref{app:seed-init} details the cycle library.

\subsection{Evidence Graph Construction}
\label{sec:graph}
Starting from $\mathcal G_0$, the enricher and expander construct the
evidence graph under depth limit $D$ and node budget $B$. The enricher
collects, derives, and verifies facts, while the expander selects facts
and promotes referenced entities for further expansion.
Algorithm~\ref{alg:construction} in Appendix~\ref{app:graph-expansion}
gives the procedure.

\paragraph{Multi-source enrichment and derived facts.}
To reduce evidence co-coverage, the enricher collects facts from
heterogeneous sources, including Wikidata, open web pages, structured
databases, Google Scholar, and Google Maps, and distributes root facts
across sources where possible. Non-root nodes sample sources for efficiency,
since only one or two facts are typically selected per
node.

The enricher also constructs derived facts from multiple atomic facts or
source texts, making the resulting relations less likely to appear
verbatim in a single retrieved item~\citep{shen2025hopweaver}. For example,
separately retrieved records that artists $A$ and $B$ performed at venue
$V$ yield the relation that they performed at the same venue. The four
constructors are coincidence bridging, count aggregation, numerical
relation, and meta-fact extraction; Appendix~\ref{app:derived} gives their
definitions and examples.

\paragraph{Fact verification and selection.}
Candidate facts are checked for source and entity consistency
\citep{cui2025toward}, rejecting conflicts with retrieved evidence,
accepted graph facts, or mismatched entities. From the verified pool,
an LLM-based selector favors facts that are weakly identifying in isolation
but useful for further expansion. Their combination remains informative
while reducing single-clue shortcuts.
Appendix~\ref{app:fact-selection} provides details and examples.

\subsection{Question Formulation}
\label{sec:projection}

The formulation agent selects an answer node $y^\star$ and subgraph
$\mathcal G^*\subseteq\mathcal G$, removes redundant facts, and verbalizes
clues that are jointly identifying but individually generic. Intermediate entity names are replaced with referring expressions
such as ``the artist,'' ``the institution,'' or ``the work,''
limiting direct access to queries that require those names.

For remaining names and literal values, \emph{exact-value fuzzing}
replaces an explicit constant with a truthful, less directly searchable
description~\citep{zhang2025infoagent}. For example, 1863 can become
``the second half of the nineteenth century.'' The agent chooses among
category generalization, range relaxation, meta-attribute description,
arithmetic encoding, and contrastive exclusion according to the value's
semantic type, distinctiveness, and role in the question. The complete
question should preserve the same gold answer and remain verifiable.
Appendix~\ref{app:fuzzing} details the five transformations.

Following PostTrainBench~\citep{rank2026posttrainbench}, we apply n-gram
filtering and remove flagged examples, totaling approximately $0.01\%$ of the data.

\subsection{Adversarial Refinement}
\label{sec:refinement}

We evaluate each draft using Section~\ref{sec:diagnostics}. Drafts below
the retrieval-cost threshold $\tau_{\mathrm{th}}$ are refined as
shortcut-prone, while unsolved drafts are repaired for ambiguity or
excessive fuzziness. The threshold triggers refinement but does not filter
the final dataset.

For shortcut-prone drafts, we replace co-covered evidence, remove overly
selective facts, or further withhold or fuzz exposed constants
\citep{anantha2025nanoflux,kazoom2025vault}. Answer-before-evidence cases
trigger root or evidence-path revision. For unsolved drafts, we narrow
over-fuzzed clues, remove ambiguity, or restore necessary constraints.
Revised drafts are reevaluated under the same search setting.
Appendix~\ref{app:refinement-cases} provides complete examples.
\FloatBarrier
\section{Experiments}
\label{sec:experiments}
\begin{table}[t]
\caption{Performance comparison on search benchmarks. Overall denotes the average score across available benchmarks.
Bold values indicate the best results within each group.}
\label{tab:main_results}
\centering
\normalsize
\setlength{\tabcolsep}{3.5pt}
\renewcommand{\arraystretch}{1.0}

\begin{tabular}{@{}lccccc@{}}
\toprule
\textbf{Model}
& \textbf{BrowseComp}
& \textbf{BC-ZH}
& \textbf{xbench}
& \textbf{WideSearch}
& \textbf{Overall}
\\
\midrule

\multicolumn{6}{c}{\textit{Large-scale Agents}}\\
\midrule

DeepSeek-V4-Pro-0813
& 84.4
& \textbf{85.3}
& 84.1
& 79.1
& 83.2
\\

DeepSeek-V4-Flash-0731
& 80.6
& 81.0
& \textbf{89.0}
& 80.6
& 82.8
\\

GLM-5.3
& 75.3
& 74.1
& 83.3
& 61.9
& 73.7
\\

Kimi-K3
& \textbf{91.2}
& 81.7
& 86.3
& 78.3
& \textbf{84.4}
\\

Qwen3.8-Max
& 80.5
& 82.1
& 84.9
& \textbf{81.6}
& 82.3
\\

\midrule

\multicolumn{6}{c}{\textit{30B-scale Agents}}\\
\midrule

REDSearcher
& 57.4
& 58.2
& -
& -
& -
\\

Apodex-1.0-mini
& 71.5
& 80.6
& 80.0
& 67.3
& 74.9
\\

BigBang-v1
& 76.5
& 69.0
& 86.3
& 50.2
& 70.5
\\

Qwen3.6-35B-A3B
& 62.9
& 63.2
& 82.7
& 62.3
& 67.8
\\

\addlinespace[2pt]

\textbf{FORT-Searcher-35B-A3B}
& \textbf{83.3}
& \textbf{83.6}
& \textbf{87.0}
& \textbf{74.1}
& \textbf{82.0}
\\

\midrule

\multicolumn{6}{c}{\textit{9B-scale Agents}}\\
\midrule

Qwen3.5-9B
& 57.4
& 55.4
& 78.9
& 58.4
& 62.5
\\

\textbf{FORT-Searcher-9B}
& \textbf{75.2}
& \textbf{75.6}
& \textbf{82.7}
& \textbf{63.5}
& \textbf{74.3}
\\

\midrule

\multicolumn{6}{c}{\textit{4B-scale Agents}}\\
\midrule

Qwen3.5-4B
& 53.2
& 40.3
& 78.0
& 53.3
& 56.2
\\

\textbf{FORT-Searcher-4B}
& \textbf{74.4}
& \textbf{74.6}
& \textbf{81.0}
& \textbf{57.9}
& \textbf{72.0}
\\

\bottomrule
\end{tabular}%

\end{table}
\subsection{Experimental setup}
\label{sec:setup}
We train \textsc{FORT}-Searcher with SFT on search trajectories
generated from \textsc{FORT} tasks using DeepSeek-V4-Pro-0813 as the
teacher. Qwen3.6-35B-A3B~\citep{qwen36_35b_a3b} is the primary
backbone, with Qwen3.5-9B and Qwen3.5-4B~\citep{qwen3.5} used
for scaling analysis.

We evaluate on BrowseComp~\citep{wei2025browsecomp},
BrowseComp-ZH~\citep{zhou2025browsecomp2},
xbench-DeepSearch~\citep{chen2025xbench}, and
WideSearch~\citep{wong2026widesearch}.
We compare against large-scale and 30B-scale search agents,
following official evaluation protocols when available and
otherwise using the same evaluation harness and budget.
Full training, evaluation, and context-management settings are
provided in Appendix~\ref{app:exp_setup}.
\subsection{Main Results}
\label{sec:main_results}

Table~\ref{tab:main_results} compares \textsc{FORT}-Searcher with
existing search agents on four benchmarks. \textsc{FORT}-Searcher-35B-A3B
achieves the highest overall score among the 30B-scale agents, while the
9B and 4B variants consistently outperform their corresponding Qwen3.5
backbones, showing that \textsc{FORT}-generated trajectories provide
effective supervision.

For Table~\ref{tab:main_results}, context management is enabled for
BrowseComp and BrowseComp-ZH and disabled for xbench-DeepSearch and
WideSearch. Figure~\ref{fig:cm_scales} further disentangles training from
context management by evaluating both the original backbones and
\textsc{FORT}-Searcher without context management. Even without context
management, \textsc{FORT} training improves BrowseComp by 5.3--18.9 points
and BrowseComp-ZH by 16.4--25.8 points across model scales. Context
management provides additional gains to \textsc{FORT}-Searcher at all
three scales~\citep{tang2025beyond,yen2510lost}.

Since most baselines do not fully disclose their training data, dataset
sizes, or training procedures, a fully controlled system-level comparison
is not feasible. To improve comparability, we compare models under matched
inference settings and budgets in Figure~\ref{fig:cm_budget}, and
additionally train \textsc{FORT}-Searcher on the same backbones as
corresponding baselines in Appendix~\ref{app:additional_backbones}.
Under matched inference budgets, \textsc{FORT}-Searcher achieves the
strongest performance on BrowseComp at medium-to-high retrieval budgets. Under
same-backbone comparisons, its advantage also persists across backbone
variants. These complementary controls support the robustness of the
observed gains.
\begin{figure}[t]
    \centering
    \begin{subfigure}[t]{0.495\linewidth}
        \centering
        \includegraphics[width=0.99\linewidth]
            {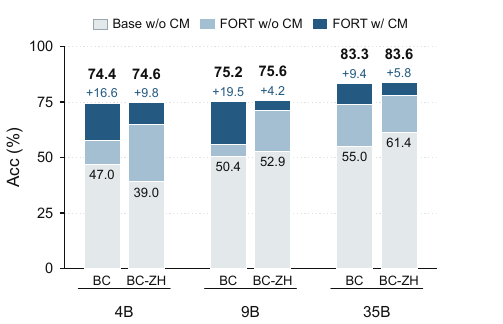}
        \caption{Training and context management.}
        \label{fig:cm_scales}
    \end{subfigure}
    \hfill
    \begin{subfigure}[t]{0.495\linewidth}
        \centering
        \includegraphics[width=0.99\linewidth]
            {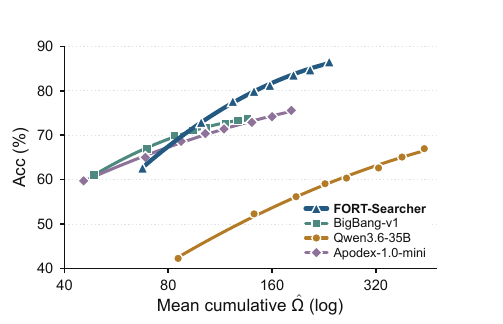}
        \caption{BrowseComp accuracy versus $\widehat{\Omega}$.}
        \label{fig:cm_budget}
    \end{subfigure}
    \caption{\textsc{FORT} training and context-management effects, and
BrowseComp retrieval-budget scaling.}
    \label{fig:context_management}
\end{figure}
\subsection{Further Analysis}
\subsubsection{Dataset Difficulty Comparison}
\label{sec:data_dif}

\paragraph{Dataset-level difficulty.}
To compare shortcut-related realized search behavior across deep-search
datasets, we compare \textsc{FORT} with InfoSeek~\citep{xia2025open},
DeepResearch-9K~\citep{wu2026deepresearch},
MiroVerse-Voyager~\citep{team2025mirothinker},
DeepDive~\citep{lu2025deepdive}, OpenSeeker~\citep{du2026openseeker},
and REDSearcher~\citep{chu2026redsearcher}.
For each dataset, we randomly sample 200 questions and evaluate them with
DeepSeek-V4-Pro-0813 under the same retrieval settings and budget. Accuracy is
computed over all 200 questions, while trajectory diagnostics are computed
over successful runs. These diagnostics measure realized search behavior
under the evaluated solver and retrieval interface rather than the oracle
minimum identifying-route cost. To assess solver dependence, we repeat the
analysis with the weaker Qwen3.5-397B-A17B~\citep{qwen3.5},
with the same overall pattern reported in Appendix~\ref{app: Cross-Solver}.

Figure~\ref{fig:dataset_difficulty} shows that \textsc{FORT} yields the
highest solving cost and latest answer exposure under the same diagnostic
setting. Compared with REDSearcher, \textsc{FORT} maintains comparable
solver accuracy while exhibiting both higher $\widehat{\Omega}$ and later
$T_{\mathrm{hit}}$, showing that higher retrieval effort is accompanied by
later answer exposure rather than simply lower solvability. Meanwhile,
$\widehat{p}_{\mathrm{prior}}$ remains comparable.

Since $T_{\mathrm{hit}}$ and $\widehat{p}_{\mathrm{prior}}$ are
success-conditioned, we additionally examine $\widehat{\Omega}$ on
failed runs. Compared with REDSearcher, the strongest
baseline under the same diagnostic setting, \textsc{FORT} remains more
costly on failed trajectories, averaging $169.5$ versus $135.2$ retrieval
calls. Thus, the higher retrieval effort is not confined to successful
trajectories. Appendix~\ref{app:data_comp} further shows lower single-clue selectivity,
more dispersed evidence, and higher realized pre-answer dependency costs.

\begin{figure}[t]
    \centering
    \begin{subfigure}[t]{0.495\linewidth}
        \centering
        \includegraphics[width=0.99\linewidth]
            {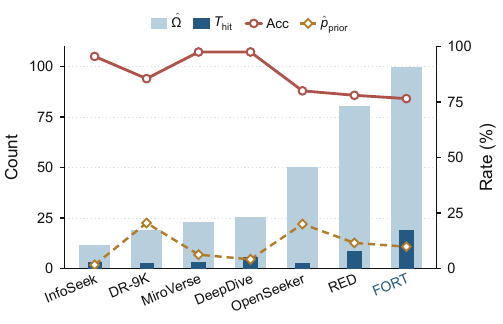}
        \caption{Dataset comparison.}
        \label{fig:dataset_difficulty}
    \end{subfigure}
    \hfill
    \begin{subfigure}[t]{0.495\linewidth}
        \centering
        \includegraphics[width=0.99\linewidth]
            {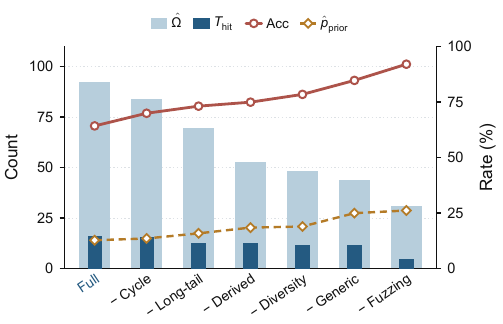}
        \caption{Ablation on shortcut-resistant components.}
        \label{fig:ab_short}
    \end{subfigure}
    \caption{Solver accuracy and trajectory diagnostics for dataset comparison
and cumulative ablation. For dataset comparison, accuracy is computed on
200 randomly sampled questions per source and trajectory diagnostics on
the successful subset. For ablation, accuracy uses all questions and
trajectory diagnostics use successful trajectories. Each ablation removes
one additional component.}
    \label{fig:difficulty_analysis}
\end{figure}

\paragraph{Training-data difficulty.} To examine whether FORT's training value can be
explained by retrieval cost alone, we construct four 3K-trajectory training
sets. Because matching distributions across synthesis pipelines is infeasible,
we improve comparability by keeping the teacher, search environment, trajectory count, distillation pipeline,
backbone, and SFT recipe fixed. Three sets are sampled
from the pooled collection of six open-source datasets above, with average
solving costs of 30, 65, and 100 retrieval calls. The fourth uses FORT data
with a matched $\widehat{\Omega}=100.0$. As shown in
Table~\ref{tab:difficulty-ablation}, increasing the average retrieval cost of
open-source training data from 30.0 to 100.0 improves BrowseComp from 55.7
to 60.7. At the matched retrieval cost, training on FORT further improves
performance to 63.9, showing that retrieval cost alone does not explain the
observed gap.

\begin{table*}[!t]
\centering

\begin{tabular}{@{}p{0.48\textwidth}@{\hspace{0.04\textwidth}}p{0.48\textwidth}@{}}

\begin{minipage}[t]{\linewidth}
\vspace{0pt}
\caption{Training results under different settings.}
\label{tab:difficulty-ablation}
\end{minipage}
&
\begin{minipage}[t]{\linewidth}
\vspace{0pt}
\caption{Effect of adversarial refinement.}
\label{tab:adv_refinement}
\end{minipage}
\\[0pt]

\begin{minipage}[t]{\linewidth}
\vspace{0pt}
\centering
\normalsize
\setlength{\tabcolsep}{3pt}
\renewcommand{\arraystretch}{1.0}

\begin{tabular*}{\linewidth}{@{\extracolsep{\fill}}cccc@{}}
\toprule
\multicolumn{3}{c}{\textbf{Data settings}}
& \multicolumn{1}{c}{\textbf{Result}} \\
\cmidrule(lr){1-3}
\cmidrule(lr){4-4}

$\boldsymbol{\widehat{\Omega}}$
& $\boldsymbol{T_{\mathrm{hit}}}$
& $\boldsymbol{\widehat{p}_{\mathrm{prior}}(\%)}$
& \textbf{BrowseComp} \\
\midrule

30.0  & 4.3  & 8.7  & 55.7 \\
65.0  & 7.6  & 13.4 & 59.7 \\
100.0 & 10.7 & 16.1 & 60.7 \\
\midrule

\textbf{100.0}
& \textbf{18.9}
& \textbf{9.2}
& \textbf{63.9} \\
\bottomrule
\end{tabular*}

\end{minipage}
&

\begin{minipage}[t]{\linewidth}
\vspace{0pt}
\centering
\normalsize
\setlength{\tabcolsep}{2pt}
\renewcommand{\arraystretch}{1.0}

\begin{tabular*}{\linewidth}{@{\extracolsep{\fill}}llccc@{}}
\toprule
\multicolumn{2}{c}{\textbf{Draft settings}}
& \multicolumn{3}{c}{\textbf{Results}} \\
\cmidrule(lr){1-2}
\cmidrule(lr){3-5}

\textbf{Draft type}
& \textbf{Version}
& $\boldsymbol{\widehat{\Omega}}$
& $\boldsymbol{\overline{T}_{\mathrm{hit}}}$
& $\boldsymbol{\widehat{p}_{\mathrm{prior}}(\%)}$ \\
\midrule

\multirow{2}{*}{\makecell[l]{Shortcut-\\prone}}
& Original
& 28.7
& 5.4
& 15.7 \\

& Refined
& \textbf{55.6}
& \textbf{12.1}
& \textbf{12.4} \\
\midrule

\multirow{2}{*}{\makecell[l]{Initially\\unsolved}}
& Original
& --
& --
& -- \\

& Refined
& 87.2
& 21.2
& 1.7 \\
\bottomrule
\end{tabular*}

\end{minipage}

\end{tabular}
\end{table*}
\subsubsection{Adversarial Refinement}
To evaluate whether adversarial refinement calibrates question difficulty,
we consider two failure modes: \emph{shortcut-prone} drafts with early
identifying routes or prior binding, and \emph{initially unsolved} drafts
affected by excessive fuzziness or ambiguity. We refine both groups and
rerun the adversary under the same search setting. Refinement resolves
70\% of shortcut-prone drafts and yields correct solutions for 40\% of
initially unsolved drafts, measured over all submitted drafts in each group.

As shown in Table~\ref{tab:adv_refinement}, shortcut-prone drafts increase
from $\widehat{\Omega}=28.7$ to $55.6$ and from $T_{\mathrm{hit}}=5.4$ to
$12.1$, while $\widehat{p}_{\mathrm{prior}}$ decreases from $15.7$ to
$12.4$. Refined initially unsolved drafts achieve $\widehat{\Omega}=87.2$
and $T_{\mathrm{hit}}=21.2$, indicating harder realized search behavior
under the evaluated adversarial search setting. To assess whether increased difficulty comes at the expense
of task validity, we manually audit 150 \textsc{FORT} questions. All final groups
retain at least 94\% answer correctness and 92\% source support. Questions
accepted without refinement achieve 94\% uniqueness and 92\% unambiguous
formulation, while refinement of initially unsolved drafts improves these
rates from 72\% to 90\% and from 74\% to 88\%, respectively.
Appendix~\ref{app:validity_audit} provides the full protocol and results.
\subsubsection{Shortcut-Resistance Ablation}

To examine the effect of progressively removing shortcut-resistant
controls, we conduct a cumulative ablation on 3000 questions. Each
configuration removes one additional component from the preceding one
under the same evaluation setting.

Figure~\ref{fig:ab_short} shows consistent shifts toward easier realized search behavior
as controls are removed: accuracy rises from 64.2 to 92.0,
while $\widehat{\Omega}$ decreases from 92.3 to 30.7,
$T_{\mathrm{hit}}$ from 16.1 to 4.5, and
$\widehat{p}_{\mathrm{prior}}$ increases from 12.7 to 26.1.
Under matched training settings, BrowseComp performance also decreases
from 61.9 for the full configuration to 58.1 and 54.9 for two
representative cumulative ablations, indicating that configurations with fewer shortcut-resistant controls
also provide less effective training data. Full downstream results are reported in
Appendix~\ref{app:downstream_ablation}.

Because the ablation is cumulative, these results characterize the joint
effect of the controls. We therefore additionally evaluate isolated
removals of long-tail selection, derived facts, and fuzzing in
Appendix~\ref{app:downstream_ablation}, providing
component-level evidence for three representative controls.
\FloatBarrier
\section{Related Work}
\label{sec:related}
\paragraph{Deep-search task synthesis.}
Recent work synthesizes multi-step search tasks from knowledge graphs,
web evidence, and agent exploration
~\citep{lu2025deepdive,xia2025open,du2026openseeker,
wu2026webdancer,li2025websailor}.
QUEST, SearchArt, and ORBIT further improve the scale and verification of
synthetic search data
~\citep{xie2026quest,mei2026searchart,thakur2026orbit}.
Recent methods also place more emphasis on controlling shortcuts during
construction. LoHoSearch explicitly controls candidate-space size and
structural complexity~\citep{zhao2026lohosearch}, while EvoBrowseComp removes
redundant paths from the constructed reasoning graph
~\citep{wang2026evobrowsecomp}. SearchArt further reduces direct retrieval
and parametric shortcuts through long-tail entities and dispersed evidence
~\citep{mei2026searchart}. These approaches make the intended reasoning
structure harder to bypass, but they mainly characterize difficulty from
properties of the constructed task. In contrast, \textsc{FORT} asks whether
this intended difficulty survives once the task is exposed through a concrete
retrieval interface. It explicitly targets low-cost identifying routes caused
by evidence co-coverage, selective clues, exposed constants, and model priors.

\paragraph{Search difficulty and execution.}
Observed search difficulty depends on the environment, retrieval interface,
and agent execution. SearchGym and DeepSearch-World study training in verifiable
search environments~\citep{zhang2026searchgym,geng2026deepsearch}.
SIEVE shows that interface expressiveness changes evidence access
~\citep{wang2026search}, while dynamic context management affects performance
over long search trajectories~\citep{tang2025beyond}. Trajectory analyses show
that search length alone reveals neither when useful evidence is acquired nor
whether it is effectively used~\citep{liu2026diagnosing}.
These works primarily study system or agent effects on performance.
Under a concrete solver and retrieval setting, \textsc{FORT} uses solving cost,
answer-exposure timing, and prior-answer behavior to characterize realized
search and identify and repair residual shortcuts in synthesized questions.
\section{Conclusion}
\label{sec:conclusion}
We introduce a shortcut-aware difficulty framework and FORT, targeting
four shortcut risks throughout synthesis and adversarial refinement.
Under the evaluated settings, FORT shows later answer exposure, lower
single-clue selectivity, greater evidence dispersion, and higher realized
pre-answer dependency costs than open-source datasets. SFT on FORT trajectories yields the strongest overall performance among
the compared 30B-scale agents and retains an advantage at matched retrieval
cost, with the same overall pattern across diagnostic solvers.
\FloatBarrier
\label{maintext:lastpage}
\clearpage
\section*{AI Use Statement}
Generative language models are used within the proposed research pipeline
for task synthesis, adversarial solving, and trajectory-based analysis.
An AI assistant also contributed to literature collection and manuscript revision. The authors manually reviewed all AI-assisted content and remain fully
responsible for the final manuscript.

\bibliography{references}
\clearpage
\appendix
\section{Difficulty Framework and Proofs}
\label{app:formal}

This appendix consolidates the definitions and proofs underlying
Section~\ref{sec:shortcuts}. We define the evidence-only reference solver and
valid identifying routes, prove the structural and component lower bounds,
and establish the one-query collapse case. We then relate the reference
cost to the cost of a concrete solver.

\subsection{Task and identifying subsets}
\label{app:formal-task}

We use the task shorthand $q=(\mathcal X,\mathcal C_q,\Sigma)$ and answer sets
$\mathrm{Ans}(\mathcal P)$ defined in
Equations~\eqref{eq:problem_tuple}--\eqref{eq:answer}.
The full constraint set uniquely identifies the gold answer:
$\mathrm{Ans}(\mathcal C_q)=\{y^\star\}$. A subset is
\emph{identifying} if it already determines this answer. Write
\begin{equation}
\mathcal I_q
=\{\mathcal P\subseteq\mathcal C_q:
\mathrm{Ans}(\mathcal P)=\{y^\star\}\}
\label{eq:app_identifying}
\end{equation}
for the family of identifying subsets. A solver need not verify every
constraint in the question if it can acquire evidence for a subset in
$\mathcal I_q$.

\subsection{Evidence-only reference solver}
\label{app:formal-policy}

\paragraph{Queries and verification.}
Let $\Theta_\Sigma$ denote the queries supported by the retrieval
interface, and let $\mathrm{Retrieve}_\Sigma(\theta)$ denote the evidence
returned by query $\theta$. Evidence items may be pages, snippets, or
records. For accumulated evidence $E$, write
\begin{equation}
V_E
=\{(x,c)\in\mathcal X\times\mathcal C_q:
E\text{ verifies }c(x)=1\}.
\label{eq:app_verification}
\end{equation}
Thus, $V_E$ records the candidate--constraint pairs verified by the
retrieved evidence.

\paragraph{Executable queries.}
Let $h$ contain the question and preceding queries and observations.
Write $\mathcal K(h)$ for the constants available from the question and
previously retrieved evidence, and $\mathrm{Const}(\theta)$ for those
needed to instantiate query $\theta$. The executable queries are
\begin{equation}
\mathrm{Exec}(h)
=\{\theta\in\Theta_\Sigma:
\mathrm{Const}(\theta)\subseteq\mathcal K(h)\}.
\label{eq:app_exec}
\end{equation}
A reference policy issues only queries in $\mathrm{Exec}(h)$.
For example, a query using an advisor's name requires that name to be
available from the question or earlier evidence.

\paragraph{No-guessing termination.}
A reference policy may return $x$ only after its accumulated evidence
verifies an identifying subset for $x$. Formally, there must exist
$\mathcal P\subseteq\mathcal C_q$ such that
\begin{equation}
\mathrm{Ans}(\mathcal P)=\{x\},
\qquad \{x\}\times\mathcal P\subseteq V_E.
\label{eq:app_termination}
\end{equation}
Since $y^\star$ satisfies every constraint in $\mathcal C_q$, it also
satisfies every constraint in $\mathcal P$. The singleton condition
therefore implies $x=y^\star$.

\paragraph{No problem-specific prior knowledge.}
The policy may use basic reading and composition over retrieved evidence,
but not entity-level prior knowledge specific to the question. In
particular, it cannot name the target from memory or use a memorized
target profile to bypass evidence acquisition.

Let $\Pi_{\mathrm{post}}$ denote the policies satisfying these
restrictions and having nonzero probability of successfully solving
$q$. We use $|\tau|$ to denote the total number of retrieval calls in
trajectory $\tau$, counting each retrieval invocation separately.
The pure-posterior cost is
\begin{equation}
D_{\mathrm{post}}(q)
=
\inf_{\pi\in\Pi_{\mathrm{post}}}
\mathbb{E}_{\tau\sim\pi}
\left[
|\tau|
\mid
\tau \text{ succeeds}
\right].
\label{eq:dpost}
\end{equation}
It measures the minimum expected retrieval cost over
successful trajectories of the no-prior, no-guessing reference solver.
An empty feasible-set minimum or infimum is assigned $+\infty$.
\subsection{Valid routes and the structural lower bound}
\label{app:formal-routes}

For an evidence-acquisition route
$\boldsymbol{\theta}=(\theta_1,\ldots,\theta_m)$,
its cost is the number of retrieval calls,
$|\boldsymbol{\theta}|=m$. For $\mathcal P\in\mathcal I_q$, an evidence-acquisition route is an
ordered query sequence
$\boldsymbol\theta=(\theta_1,\ldots,\theta_m)$.
Let $h_{j-1}$ contain the question and observations preceding query
$\theta_j$. The route is valid for $(y^\star,\mathcal P)$ if every
query is executable when issued,
\begin{equation}
\theta_j\in\mathrm{Exec}(h_{j-1}),\qquad j=1,\ldots,m,
\label{eq:app_route_executable}
\end{equation}
and the accumulated evidence verifies the subset:
\begin{equation}
\{y^\star\}\times\mathcal P\subseteq V_E,
\qquad
E=\bigcup_{j=1}^{m}\mathrm{Retrieve}_\Sigma(\theta_j).
\label{eq:app_route_verification}
\end{equation}
Define the minimum route cost as
\begin{equation}
Q_\Sigma(\mathcal P)
=\min_{\boldsymbol\theta\ \mathrm{valid\ for}\ (y^\star,\mathcal P)}
|\boldsymbol\theta|.
\label{eq:app_route_cost}
\end{equation}
Unlike a realized trajectory $\tau$, this route is an ideal reference for
a fixed target and constraint subset. A solver searching adaptively may
issue additional exploratory or redundant queries. Minimizing over
identifying subsets gives
\begin{equation}
Q_\Sigma^\star
=\min_{\mathcal P\in\mathcal I_q}Q_\Sigma(\mathcal P).
\label{eq:app_qstar}
\end{equation}
The minimum concerns all identifying subsets, rather than only the
intended route in the construction graph.

\begin{proposition}[Structural lower bound]
\label{prop:app_structural}
Under the reference-solver restrictions above,
\begin{equation}
D_{\mathrm{post}}(q)\geq Q_\Sigma^\star.
\label{eq:app_structural_bound}
\end{equation}
\end{proposition}
\begin{proof}
Consider a trajectory generated by a policy
$\pi\in\Pi_{\mathrm{post}}$ that returns an answer. By the no-guessing
restriction, it returns $y^\star$ only after retrieving evidence that
verifies some $\mathcal P\in\mathcal I_q$. Every query is executable
when issued. Its query sequence is therefore a valid route for
$(y^\star,\mathcal P)$, and
\begin{equation}
|\tau|\geq Q_\Sigma(\mathcal P)
\geq\min_{\mathcal P'\in\mathcal I_q}Q_\Sigma(\mathcal P')
=Q_\Sigma^\star.
\label{eq:app_pathwise_bound}
\end{equation}
Taking the expectation over successful reference-solver trajectories
preserves this lower bound. Taking the infimum over
$\Pi_{\mathrm{post}}$ then gives
\[
D_{\mathrm{post}}(q)
=
\inf_{\pi\in\Pi_{\mathrm{post}}}
\mathbb{E}_{\tau\sim\pi}
\left[
|\tau|
\mid
\tau \text{ succeeds}
\right]
\ge
Q^\star_\Sigma.
\]
\end{proof}

\paragraph{A sufficient condition for equality.}
If $Q_\Sigma^\star<\infty$ and a policy $\pi\in\Pi_{\mathrm{post}}$ uses
exactly $Q_\Sigma^\star$ retrieval calls on its successful trajectories
almost surely, its conditional expected cost is $Q_\Sigma^\star$.
Together with Proposition~\ref{prop:app_structural}, this gives
\begin{equation}
D_{\mathrm{post}}(q)=Q_\Sigma^\star.
\label{eq:app_equality}
\end{equation}
A policy that identifies and follows a cheapest valid route without
extra retrieval satisfies this condition. A small $Q_\Sigma^\star$
indicates that a cheap identifying route exists; a particular solver
may still need additional search to discover it.

\subsection{Selectivity, evidence dispersion, and dependency depth}
\label{app:formal-factors}

\paragraph{Subset selectivity.}
The candidate breadth remaining after applying a subset of clues is
\begin{equation}
s(\mathcal P)=|\mathrm{Ans}(\mathcal P)|.
\label{eq:app_selectivity}
\end{equation}
A smaller value means greater selectivity. If a single clue leaves only
$y^\star$, that clue is itself identifying and enters the minimization
in Equation~\eqref{eq:app_qstar}. A small, non-identifying candidate pool
can also focus subsequent exploration. Selectivity does not directly
lower-bound route length; the evidence needed to verify the relevant
clues must be considered separately.

\paragraph{Evidence dispersion.}
For a fixed subset, let
\begin{equation}
M_{\mathrm{ev}}(\mathcal P)
=\min\left\{
|\boldsymbol\theta|:
\boldsymbol\theta\text{ retrieves evidence verifying }
(y^\star,\mathcal P)
\right\}.
\label{eq:app_coverage}
\end{equation}
Queries are issued through $\Sigma$, but this minimization ignores
executability: it counts the retrievals needed to obtain the evidence,
regardless of whether their required constants are available when they
are issued. Restoring executability can only restrict the feasible
routes. If one retrieval supplies evidence for several constraints,
those constraints can be verified together.

\paragraph{Dependency depth.}
A later query depends on an earlier one when it requires a name,
attribute, or intermediate fact that is absent from the question but
exposed by earlier retrieved evidence. Let
$\mathrm{depth}(\boldsymbol\theta)$ be the length of the longest such
dependency chain in a valid route.
An empty route has depth zero.
Define
\begin{equation}
\mathrm{dep}(\mathcal P)
=\min_{\boldsymbol\theta\ \mathrm{valid\ for}\ (y^\star,\mathcal P)}
\mathrm{depth}(\boldsymbol\theta).
\label{eq:app_dependency}
\end{equation}
For example, searching for an advisor's other students may first require
finding the advisor's name. If the question already supplies that name,
the downstream query becomes executable from the start and the required
dependency chain may become shorter.

\subsection{Component lower bound and one-query collapse}
\label{app:formal-components}

\begin{proposition}[Component lower bound]
\label{prop:app_components}
For every identifying subset $\mathcal P$,
\begin{equation}
Q_\Sigma(\mathcal P)
\geq\max\{M_{\mathrm{ev}}(\mathcal P),\mathrm{dep}(\mathcal P)\}.
\label{eq:app_component_bound}
\end{equation}
Consequently,
\begin{equation}
D_{\mathrm{post}}(q)\geq Q_\Sigma^\star
\geq\min_{\mathcal P\in\mathcal I_q}
\max\{M_{\mathrm{ev}}(\mathcal P),\mathrm{dep}(\mathcal P)\}.
\label{eq:app_combined_bound}
\end{equation}
\end{proposition}
\begin{proof}
If $Q_\Sigma(\mathcal P)=+\infty$, the inequality is immediate.
Otherwise, let $\boldsymbol\theta^\star$ be a shortest valid route for
$(y^\star,\mathcal P)$, so that
$|\boldsymbol\theta^\star|=Q_\Sigma(\mathcal P)$.
First, it retrieves enough evidence to verify every constraint in
$\mathcal P$ for $y^\star$. It is therefore feasible for the
minimization defining $M_{\mathrm{ev}}(\mathcal P)$, which ignores
executability. Thus,
\begin{equation}
Q_\Sigma(\mathcal P)=|\boldsymbol\theta^\star|
\geq M_{\mathrm{ev}}(\mathcal P).
\label{eq:app_coverage_bound}
\end{equation}
Second, a dependency chain is contained within the route, so the route
cannot be shorter than its longest dependency chain. The minimum depth
over all valid routes cannot exceed the depth of this particular route.
Combining these observations gives
\begin{equation}
Q_\Sigma(\mathcal P)=|\boldsymbol\theta^\star|
\geq\mathrm{depth}(\boldsymbol\theta^\star)
\geq\mathrm{dep}(\mathcal P).
\label{eq:app_depth_bound}
\end{equation}
Both bounds hold for the same route, establishing
Equation~\eqref{eq:app_component_bound}. Taking the minimum over
identifying subsets and applying Proposition~\ref{prop:app_structural}
establishes Equation~\eqref{eq:app_combined_bound}.
\end{proof}

Evidence dispersion and dependency depth constrain route cost in
different ways. The former measures evidence acquisition without
executability; the latter arises from the information needed to make
later queries executable. Selectivity affects which subsets are
identifying, rather than supplying another route-length lower bound.

\begin{corollary}[One-query collapse]
\label{cor:app_one_query}
Assume that evidence verification requires at least one retrieval.
If some $\mathcal P\in\mathcal I_q$ is verified by a single initially
executable query, then
\begin{equation}
Q_\Sigma^\star=1.
\label{eq:app_one_query}
\end{equation}
\end{corollary}
\begin{proof}
That query forms a valid route of length one, giving
\[
Q_\Sigma^\star
=\min_{\mathcal P'\in\mathcal I_q}Q_\Sigma(\mathcal P')
\leq Q_\Sigma(\mathcal P)\leq 1.
\]
Since at least one retrieval is required for verification,
$Q_\Sigma^\star\geq 1$. The two inequalities establish equality.
\end{proof}

This case illustrates the construction--search mismatch: a question
built from many facts can still have a one-query identifying route when
an initially executable query returns evidence for an identifying subset.

\subsection{Solver-specific cost and trajectory diagnostics}
\label{app:formal-solver}

A concrete solver $\pi_0$ may use prior knowledge or propose an answer
before retrieval provides support. When its success probability is positive, its expected
successful-trajectory cost is
\begin{equation}
\Omega(q,\pi_0)
=\mathbb E_{\tau\sim\pi_0}[|\tau|\mid\tau\text{ succeeds}].
\label{eq:app_solver_cost}
\end{equation}
When both $D_{\mathrm{post}}(q)$ and $\Omega(q,\pi_0)$ are finite,
define the solver-side cost difference as
\begin{equation}
U_{\pi_0}(q)=D_{\mathrm{post}}(q)-\Omega(q,\pi_0).
\label{eq:app_prior_difference}
\end{equation}
A positive value means that the solver uses fewer retrieval calls on
successful trajectories than the reference solver. Because this cost is conditioned on success, it should be interpreted together with the
solver's success rate.

Rearranging the definition gives
\begin{equation}
\Omega(q,\pi_0)=D_{\mathrm{post}}(q)-U_{\pi_0}(q).
\label{eq:app_decomposition}
\end{equation}
Under the sufficient condition for equality above, if
$U_{\pi_0}(q)=0$, then
\begin{equation}
\Omega(q,\pi_0)=D_{\mathrm{post}}(q)=Q_\Sigma^\star.
\label{eq:app_realized_equality}
\end{equation}

The theoretical quantities describe identifying routes and the
reference cost, but are not all directly computable at scale.
Section~\ref{sec:diagnostics} therefore uses observable trajectory
signatures: solving cost, answer hit time, and the answer-before-evidence indicator. Together, these diagnostics examine retrieval effort, early
answer exposure, and answer-before-evidence behavior. The mechanisms
and representative cases are discussed in
Appendix~\ref{app:shortcut_analysis}.

\section{Shortcut Risks and Diagnostic Cases}
\label{app:shortcut_analysis}
\label{app:cases}
\newcounter{shortcutcase}
\newtcolorbox{shortcutexample}[1]{
    enhanced,
    breakable=false,
    colback=white,
    colframe=black,
    colbacktitle=black,
    coltitle=white,
    boxrule=0.8pt,
    arc=3pt,
    title=#1,
    fonttitle=\bfseries,
    before skip=0pt,
    after skip=0pt
}

This appendix connects the four risks in Section~\ref{sec:risks} to
Appendix~\ref{app:formal} and illustrates their observable effects.
The first three concern evidence acquisition: which constraints suffice,
how their support is retrieved, and which queries are executable.
Prior-knowledge binding concerns information supplied by the solver itself.

The cases retain the reported questions, answer labels, trigger queries,
and trigger rounds. Descriptions of retrieved content and model behavior
are summaries, not verbatim tool or model outputs. Round~0 denotes behavior
before the first retrieval call.

\subsection{Evidence co-coverage}
\label{app:evidence_cocoverage}

Evidence co-coverage occurs when one retrieved item, page, or snippet
verifies several intended constraints at once. Although a question may
contain multiple clues, their verification need not require separate
retrievals. The same observation can expose the answer together with
several supporting facts, compressing the intended evidence-acquisition
steps.

The corresponding quantity is the relaxed verification cost
$M_{\mathrm{ev}}(\mathcal P)$ in Equation~\eqref{eq:app_coverage}.
Co-covered support can make this quantity small because fewer retrievals
are needed to verify the selected constraints. If an initially executable
query verifies an entire identifying subset, the minimum route cost
is at most one, and is exactly one when verification requires at
least one retrieval (Corollary~\ref{cor:app_one_query}).

Figure~\ref{fig:shortcut_case_cocoverage} illustrates this mechanism:
a single result surfaces Grey Investment Group together with several
answer-side facts.

\begin{figure}[!htbp]
\centering
\refstepcounter{shortcutcase}\label{case:cocoverage}
\begin{shortcutexample}{Case~\theshortcutcase: Evidence co-coverage}
\normalsize
\textbf{Question.}
A corporate entity in the Pacific region, which was formally established
to consolidate various holdings, traces its commercial lineage to
entrepreneurial activities in Apia that began in the early decades of the
1900s. These foundational ventures were launched by a prominent
individual whose career was primarily built in the local hospitality
industry. What is the formal name of the resulting holding company that
was incorporated?

\smallskip
\textbf{Answer label:} Grey Investment Group.
\quad\textbf{Trigger:} Round~1.

\smallskip
\textbf{Reported query.}\par
{\raggedright\ttfamily
Apia Samoa early 1900s holding company founded by hospitality entrepreneur\par}

\smallskip
\textbf{Explanation.}
A single returned snippet names Grey Investment Group and describes it
as a Samoan investment enterprise with tourism and hospitality activities
and assets across the Pacific region. The result exposes the answer and
these attributes before the model names the answer.

\end{shortcutexample}
\caption{Diagnostic example of evidence co-coverage.}
\label{fig:shortcut_case_cocoverage}
\end{figure}
\FloatBarrier

\subsection{Single-clue selectivity}
\label{app:single_clue_selectivity}

Single-clue selectivity occurs when one clue, or a small clue subset,
identifies the answer or narrows the candidate pool to a few plausible
answers. The solver can then focus on those candidates before acquiring
evidence for the remaining constraints.

Candidate breadth $s(\mathcal P)$ is defined in
Equation~\eqref{eq:app_selectivity}. A singleton identifying clue enters
the minimization over $\mathcal I_q$ in Equation~\eqref{eq:app_qstar};
if its evidence is cheaply accessible, it offers a cheap route. A clue
that merely narrows the candidates can instead facilitate candidate
discovery while leaving disambiguation and verification necessary.

Figure~\ref{fig:shortcut_case_singleclue} shows a query using only the
television-broadcast clue to surface the answer year, 1951, before the
remaining threads are checked.

\begin{figure}[!htbp]
\centering
\refstepcounter{shortcutcase}\label{case:singleclue}
\begin{shortcutexample}{Case~\theshortcutcase: Single-clue selectivity}
\normalsize
\textbf{Question.}
A specific calendar year in the middle of the twentieth century can be
defined by four converging threads. First, it is the year a foundational
international convention on the legal protection of displaced persons
was opened for signature, administered by a UN agency. Second, that same
UN agency has a mandate that significantly overlaps with a much larger
intergovernmental body---one that formally became a UN-related
organization more than five decades after the first human-crafted object
reached another celestial body. Third, in this identical year, on a day
in the final week of June, a major U.S. television network publicly
transmitted its first regular programming using a full-color system.
Finally, in that same twelve-month period, in a nation through which the
Paran\'{a} River flows, a documented submission grappling contest took
place between two high-profile figures from the world of martial arts.
Determine which single year satisfies all four of these conditions.

\smallskip
\textbf{Answer label:} 1951.
\quad\textbf{Trigger:} Round~1.

\smallskip
\textbf{Reported query.}\par
{\raggedright\ttfamily
first regular full-color programming US television network final week June\par}

\smallskip
\textbf{Explanation.}
The query uses only the television-broadcast thread. Its results surface
events in June 1951, after which the model treats 1951 as the principal
candidate while checking the other threads.

\end{shortcutexample}
\caption{Diagnostic example of single-clue selectivity.}
\label{fig:shortcut_case_singleclue}
\end{figure}
\FloatBarrier

\subsection{Exposed constants}
\label{app:exposed_constants}

A name, title, exact value, or distinctive phrase on the question surface
can supply a query argument that would otherwise need to be discovered.
Under Equation~\eqref{eq:app_exec}, exposing that constant makes queries
that require it executable as soon as their other arguments are available.
This can remove discovery calls and shorten dependencies as defined in
Equation~\eqref{eq:app_dependency}. A long relational description may
therefore contain a directly searchable phrase that bypasses much of
its intended traversal.

Figure~\ref{fig:shortcut_case_exposedconstant} illustrates how a
distinctive phrase supplied in the question is reused in a query that
retrieves the answer, Jack Munro.

\begin{figure}[!htbp]
\centering
\refstepcounter{shortcutcase}\label{case:constants}
\begin{shortcutexample}{Case~\theshortcutcase: Exposed constants}
\normalsize
\textbf{Question.}
Identify an individual who held a regional leadership role in a major
North American forestry labor union, an acronym that could also denote
an international water association or a wrestling organization, during
the latter half of the 1980s. This person was publicly associated with
a city whose name derives from an Indigenous word for ``grizzly bear,''
located within a tourism region in the southern interior of the
westernmost Canadian province. When the provincial government announced
new labor legislation in the mid-1980s, this individual publicly stated
that the people of this province want to work, not protest. The same city
hosts an airport that provides temperature observations used to model
evaporation rates for a glacially formed lake in a north-south oriented
valley. This individual shares the same full name as a professional
pickleball player from Austin, Texas. Who is this union leader?

\smallskip
\textbf{Answer label:} Jack Munro.
\quad\textbf{Trigger:} Round~2.

\smallskip
\textbf{Reported query.}\par
{\raggedright\ttfamily
IWA union leader Kelowna 1980s "want to work not protest"\par}

\smallskip
\textbf{Explanation.}
The model reuses the phrase ``want to work, not protest'' nearly verbatim
from the question, and the retrieved results directly name Jack Munro.

\end{shortcutexample}
\caption{Diagnostic example of an exposed-constant shortcut.}
\label{fig:shortcut_case_exposedconstant}
\end{figure}
\FloatBarrier

\subsection{Prior-knowledge binding}
\label{app:prior_binding}

Prior-knowledge binding occurs when a concrete solver associates clues
with the answer before retrieval provides support. Such knowledge is
excluded from the reference class in Appendix~\ref{app:formal-policy}.
It may reduce the solver's retrieval cost, as described by
$U_{\pi_0}(q)$ in Equation~\eqref{eq:app_prior_difference}, without
shortening the evidence-only route. The observable signature is the
answer-before-evidence indicator in
Equation~\eqref{eq:prior_shortcut_indicator}.

Figure~\ref{fig:shortcut_case_prior} shows the model proposing Cleopatra
before its first retrieval call and subsequently issuing a query that
already contains her name.

\begin{figure}[!htbp]
\centering
\refstepcounter{shortcutcase}\label{case:prior}
\begin{shortcutexample}{Case~\theshortcutcase: Prior-knowledge binding}
\normalsize
\textbf{Question.}
In a coastal town where a mechanical marvel from the age of telegraphy
still stands, the town's historical associations trace back to a Roman
conqueror. That conqueror later defied the republic by crossing a
boundary river, setting in motion conflicts that culminated in a naval
battle whose aftermath led to the demise of a legendary queen. Who was
that queen?

\smallskip
\textbf{Answer label:} Cleopatra.
\quad\textbf{Trigger:} Round~0, before the first tool call.

\smallskip
\textbf{Explanation.}
Before issuing a query, the model proposes Cleopatra as a possible answer
and links the historical clues to Caesar, the civil war, Actium, and
Cleopatra's death. A retrieved observation containing the answer appears
later, following a query that already includes her name.

\end{shortcutexample}
\caption{Diagnostic example of prior-knowledge binding.}
\label{fig:shortcut_case_prior}
\end{figure}
\FloatBarrier
\clearpage

\section{Details of FORT Task Synthesis}
\label{app:construction}

This appendix provides the graph-expansion algorithm, fact and fuzzing examples, and
refinement cases accompanying Section~\ref{sec:method}.

\subsection{Entity Selection and Graph Initialization}
\label{app:seed-init}

Root selection follows the topical coverage, long-tail preference, and
evidence-availability criteria in Section~\ref{sec:initialization}.
To initialize a seed graph, \textsc{FORT} pre-mines a cycle library from
Wikidata. It builds a pair index from Wikidata edges,
enumerates small cycles, removes duplicate cycles with the same node set,
and filters out hub-based, redundant, or mass-collaboration structures.
Retained cycles are stored in an inverted index $\mathcal I$, where
$\mathcal I[u]$ contains all retained cycles involving entity $u$.
For a root $r$, the initializer uses one cycle from $\mathcal I[r]$ as the
seed graph $\mathcal G_0$ if the set is nonempty. Otherwise, it uses the
single-node seed $(\{r\},\varnothing)$.

\subsection{Evidence Graph Expansion}
\label{app:graph-expansion}

Algorithm~\ref{alg:construction} expands the seed graph under depth limit and node-addition budget. Seed-node depths are graph distances from
the root; newly added nodes receive their parent's depth plus one. The
queue $\mathcal Q$ holds expandable nodes, $\mathcal V_{\mathrm{done}}$
records processed nodes, and $b$ counts newly added nodes. For each
selected fact $(v,\rho,w)$, an edge and child node are added when $w$ is
new and the budget permits.

\begin{algorithm}[!htbp]
\normalsize
\caption{Evidence graph construction}
\label{alg:construction}
\label{alg:graph_construction}
\begin{algorithmic}[1]
\Require initialized subgraph \(\mathcal{G}_0=(V_0,E_0)\) with root \(r \in V_0\), depth limit \(D\), node-addition budget \(B\)
\Ensure evidence graph \(\mathcal{G}\)

\State \(\mathcal{G} \gets \mathcal{G}_0\)
\State \(\delta_{\mathcal{G}}(u) \gets \mathrm{dist}_{\mathcal{G}}(r,u)\) for all \(u \in V(\mathcal{G})\)
\Comment{initialize node depth}
\State \(\mathcal{Q} \gets \{u \in V(\mathcal{G}) : \delta_{\mathcal{G}}(u) < D\}\)
\State \(\mathcal{V}_{\mathrm{done}} \gets \varnothing\)
\State \(b \gets 0\)
\Comment{number of newly added nodes}

\While{\(\mathcal{Q} \neq \varnothing\) \textbf{and} \(b < B\)}
    \State \(v \gets \arg\max_{u \in \mathcal{Q}} \delta_{\mathcal{G}}(u)\)
    \Comment{expand deepest node first}
    \State \(\mathcal{Q} \gets \mathcal{Q} \setminus \{v\}\)

    \If{\(v \in \mathcal{V}_{\mathrm{done}}\) \textbf{or} \(\delta_{\mathcal{G}}(v) \geq D\)}
        \State \textbf{continue}
    \EndIf

    \State \(\mathcal{V}_{\mathrm{done}} \gets \mathcal{V}_{\mathrm{done}} \cup \{v\}\)

    \State \(\mathcal{A}_v \gets \textsc{CollectAtomicFacts}(v)\)
    \Comment{collect facts from diverse sources}
    \State \(\mathcal{D}_v \gets \textsc{ConstructDerivedFacts}(\mathcal{A}_v)\)
    \Comment{construct facts from multiple atomic facts}
    \State \(\mathcal{P}_v \gets \mathcal{A}_v \cup \mathcal{D}_v\)
    \State \(\mathcal{F}_v \gets \textsc{VerifyFacts}(v,\mathcal{P}_v)\)
    \Comment{filter inconsistent or mismatched facts}
    \State \(\mathcal{S}_v \gets \textsc{SelectFacts}(\mathcal{F}_v)\)
    \Comment{select low-shortcut facts}

    \For{\((v,\rho,w) \in \mathcal{S}_v\)}
        \Comment{\(\rho\): selected fact, \(w\): referenced entity}
        \If{\(w \notin V(\mathcal{G})\) \textbf{and} \(b < B\)}
            \State \(\mathcal{G} \gets \mathcal{G} \cup \{(v,\rho,w)\}\)
            \State \(\delta_{\mathcal{G}}(w) \gets \delta_{\mathcal{G}}(v)+1\)
            \State \(b \gets b+1\)

            \If{\(\delta_{\mathcal{G}}(w) < D\)}
                \State \(\mathcal{Q} \gets \mathcal{Q} \cup \{w\}\)
            \EndIf
        \EndIf
    \EndFor
\EndWhile

\State \Return \(\mathcal{G}\)
\end{algorithmic}
\end{algorithm}

\FloatBarrier

\subsection{Fact Construction, Verification, and Selection}
\label{app:derived}
\label{app:fact-selection}

\paragraph{Multi-source enrichment.}
The enricher collects atomic facts from Wikidata, open web pages, structured
databases, Google Scholar, and Google Maps. It avoids drawing multiple selected
facts from the same evidence item and encourages different web facts to come
from different pages. Root-node facts are distributed across sources whenever
possible. For efficiency, non-root nodes use sampled sources, since only one
or two facts are typically selected per node~\citep{zhang2026argus}.

\paragraph{Derived fact construction.}
Derived facts combine atomic facts or inspect source texts to form constraints
that are less likely to appear verbatim in a single retrieved item.
Table~\ref{tab:derived} presents the four constructors and their
examples~\citep{shen2025hopweaver}.

\begin{table}[!htbp]
\centering
\caption{Derived fact constructors and examples.}
\label{tab:derived}
\label{tab:derived_fact_examples}
\normalsize
\renewcommand{\arraystretch}{1.17}
\setlength{\tabcolsep}{4pt}
\begin{tabularx}{\linewidth}{@{}>{\raggedright\arraybackslash}p{0.18\linewidth}>{\raggedright\arraybackslash}p{0.28\linewidth}Y@{}}
\toprule
\textbf{Constructor} & \textbf{Definition} & \textbf{Example} \\
\midrule
\textbf{Coincidence bridging} &
Match the same attribute across independently retrieved records. &
$A$ held a concert at venue $V$; $B$ also held a concert at $V$.
\par $\Rightarrow$ $A$ and $B$ held concerts at the same venue. \\
\midrule
\textbf{Count aggregation} &
Combine an entity set with independently verified attributes to derive a
counting constraint. &
The cast of work $W$ includes $A$, $B$, and $C$, each of whom has won award $M$.
\par $\Rightarrow$ $W$ features at least two winners of award $M$. \\
\midrule
\textbf{Numerical relation} &
Retrieve numerical values and apply arithmetic comparison. &
$A$ was born in year $Y_1$; $A$'s advisor $B$ was born in year $Y_2$.
\par $\Rightarrow$ $A$'s advisor is $Y_1-Y_2$ years older than $A$. \\
\midrule
\textbf{Meta-fact extraction} &
Inspect, count, or summarize the original text of a rich-content entity. &
From the lyrics of song $A$:
\par $\Rightarrow$ Word $B$ appears at least three times in $A$.
\par From Chapter~3 of novel $C$:
\par $\Rightarrow$ The chapter begins with a memory-oriented passage. \\
\bottomrule
\end{tabularx}
\end{table}

\paragraph{Fact verification.}
Candidate facts are checked for source consistency and entity consistency.
Source consistency rejects a fact if it is contradicted within its evidence
item, contradicted by other retrieved evidence, or incompatible with accepted
facts in the graph. Entity consistency requires the evidence to concern the
same real-world entity as the target node. Table~\ref{tab:entity-validity}
illustrates four common failures. These checks assess factual validity
before fact selection~\citep{cui2025toward}.

\begin{table}[!htbp]
\centering
\caption{Entity-consistency failures considered during fact verification.}
\label{tab:entity-validity}
\label{tab:entity_consistency}
\normalsize
\renewcommand{\arraystretch}{1.17}
\setlength{\tabcolsep}{4pt}
\begin{tabularx}{\linewidth}{@{}>{\raggedright\arraybackslash}p{0.21\linewidth}>{\raggedright\arraybackslash}p{0.32\linewidth}Y@{}}
\toprule
\textbf{Failure mode} & \textbf{Definition} & \textbf{Example} \\
\midrule
\textbf{Similar names} &
The page describes a different entity with a similar or partially overlapping
name. &
\textbf{Target:} \emph{A Chemical Plant}.
\par \textbf{Rejected:} \emph{A Organic Chemical Plant}, a different company. \\
\midrule
\textbf{Abbreviation ambiguity} &
The same abbreviation refers to another subject. &
\textbf{Target:} a research institute abbreviated as \emph{IIS}.
\par \textbf{Rejected:} \emph{IIS} referring to \emph{Internet Information Services}. \\
\midrule
\textbf{Temporal or geographic drift} &
The surface name matches, but the time, location, or entity instance differs. &
\textbf{Target:} the $N$-th Director of Bureau $X$.
\par \textbf{Rejected:} the $(N-1)$-th Director, or a same-named bureau in another region. \\
\midrule
\textbf{Series--edition mismatch} &
The page describes a different granularity or edition. &
\textbf{Target:} \emph{The Voice of China Season~3}.
\par \textbf{Rejected:} the series-level page or pages about other seasons. \\
\bottomrule
\end{tabularx}
\end{table}

\paragraph{Generic fact selection.}
Because exact candidate breadth $s(\{c\})$ is not available in
open-domain search, we use an LLM-based judgment as a generation-time surrogate for single-clue selectivity. For each verified candidate fact, the selector considers the target entity, the fact, and its supporting evidence, and judges whether the fact is overly distinctive or strongly narrows the target in isolation. Facts judged broadly applicable while remaining useful for graph expansion are preferred. This heuristic is
used only during generation; realized selectivity is evaluated separately through the trajectory-level analysis in Appendix~\ref{app:data_comp}.
Table~\ref{tab:generic-facts} contrasts generic and representative facts for the same entities.

\begin{table}[!htbp]
\centering
\caption{Generic and representative facts for the same entities.}
\label{tab:generic-facts}
\label{tab_generic_fact_examples}
\normalsize
\renewcommand{\arraystretch}{1.17}
\setlength{\tabcolsep}{4pt}
\begin{tabularx}{\linewidth}{@{}>{\raggedright\arraybackslash}p{0.19\linewidth}YY@{}}
\toprule
\textbf{Entity} & \textbf{Generic fact} & \textbf{Representative fact} \\
\midrule
\textbf{Marie Curie} &
Worked at a university in France. &
Won Nobel Prizes in two different scientific fields. \\
\midrule
\textbf{The Eiffel Tower} &
Is located in a European capital city. &
A landmark iron tower in Paris. \\
\midrule
\textbf{The Matrix} &
Is an English-language science fiction film. &
Depicts humans living inside a simulated reality created by machines. \\
\bottomrule
\end{tabularx}
\end{table}
\FloatBarrier

\subsection{Question Formulation and Fuzzing}
\label{app:fuzzing}

The formulation agent selects an answer-bearing subgraph
$\mathcal G^*\subseteq\mathcal G$, chooses the answer node $y^\star$, and
removes redundant facts. It retains clues that are jointly identifying while
remaining individually generic. Intermediate entity names are withheld and
replaced with referring expressions such as ``the artist'', ``the institution'',
or ``the work''.

For remaining names and literal values, exact-value fuzzing produces truthful
but less directly searchable constraints while preserving the answer to the
complete question~\citep{zhang2025infoagent}. Table~\ref{tab:fuzzing} gives the
five strategies and examples. The agent chooses a strategy according to the
value's semantic type, distinctiveness, and role in the question.

\begin{table}[!htbp]
\centering
\caption{Exact-value fuzzing strategies and examples.}
\label{tab:fuzzing}
\label{tab:fuzzing_strategies}
\normalsize
\renewcommand{\arraystretch}{1.16}
\setlength{\tabcolsep}{4pt}
\begin{tabularx}{\linewidth}{@{}>{\raggedright\arraybackslash}p{0.21\linewidth}>{\raggedright\arraybackslash}p{0.28\linewidth}Y@{}}
\toprule
\textbf{Strategy} & \textbf{Description} & \textbf{Examples} \\
\midrule
\textbf{Category generalization} &
Replace an entity name with a higher-level category, type, or domain. &
\emph{International Monetary Fund} $\rightarrow$ an international financial institution.
\par\smallskip
\emph{Harvard University} $\rightarrow$ an Ivy League university in the United States. \\
\midrule
\textbf{Range relaxation} &
Replace an exact number with an interval, approximate range, or open-ended constraint. &
\emph{1863} $\rightarrow$ the second half of the nineteenth century.
\par\smallskip
\emph{16,000} $\rightarrow$ more than ten thousand. \\
\midrule
\textbf{Meta-attribute description} &
Describe a value's intrinsic form or surface structure. &
\emph{September 9} $\rightarrow$ a date whose month and day use the same number.
\par\smallskip
\emph{Hannah} $\rightarrow$ a person whose given name is a palindrome. \\
\midrule
\textbf{Arithmetic encoding} &
Express a number through arithmetic, digit-level, or number-theoretic constraints. &
\emph{42 years old} $\rightarrow$ an age that is a multiple of six.
\par\smallskip
\emph{17,921} $\rightarrow$ a five-digit prime number whose digits sum to 20. \\
\midrule
\textbf{Contrastive exclusion} &
Exclude alternatives while leaving the target value implicit. &
\emph{Canada} $\rightarrow$ a country in North America that does not border Mexico.
\par\smallskip
\emph{Wednesday} $\rightarrow$ a weekday that is not Monday. \\
\bottomrule
\end{tabularx}
\end{table}
\FloatBarrier

\subsection{Adversarial Refinement Examples}
\label{app:refinement-cases}

Adversarial refinement uses the solver's trajectory to identify clues that
expose an answer too early or prevent the question from being resolved.
Figure~\ref{fig:adv_case_broadening} illustrates shortcut repair by broadening
an aviation reference and relaxing a numerical constraint.
Figure~\ref{fig:adv_case_narrowing} illustrates ambiguity repair by narrowing
over-fuzzed clues. The before and after questions are shown in full; bold text
marks the revised spans.

\newtcolorbox{fortrefinementexample}[1]{
    enhanced,
    breakable=false,
    colback=white,
    colframe=black,
    colbacktitle=black,
    coltitle=white,
    boxrule=0.8pt,
    arc=3pt,
    title=#1,
    fonttitle=\bfseries,
    before skip=0pt,
    after skip=0pt
}

\begin{figure}[!htbp]
\centering
\begin{fortrefinementexample}{Refinement Case 1: Repairing Shortcuts}
\normalsize
\textbf{Before.}
A certain film was released in a certain year. The film has a runtime of a
three-digit minute count whose \textbf{digits sum to 10}, features protagonists
launching business schemes and a scam escalating to large monetary levels. The
film's composer is a certain individual who plays a certain instrument typically
belonging to the keyboard family.

In that same release year, a certain aviation incident occurred involving a
target being \textbf{shot down by the Soviet Union}. The incident originated
from a certain institution \textbf{having 4 runways}. Who is the director of the
film?

\smallskip
\textbf{Refinement.}
The Soviet Union reference and runway count identify an aviation incident
and hence the film's release year. Together with the runtime digit sum, these
clues sharply restrict the film candidates. The revision broadens the aviation
reference and relaxes the runtime constraint while retaining the plot, composer,
and instrument clues.

\smallskip
\textbf{After.}
A certain film was released in a certain year. The film has a runtime of a
three-digit minute count whose \textbf{digits sum to a multiple of 5}, features
protagonists launching business schemes and a scam escalating to large monetary
levels. The film's composer is a certain individual who plays a certain
instrument typically belonging to the keyboard family.

In that same release year, a certain aviation incident occurred involving a
target being \textbf{shot down by a certain country during the Cold War}. The
incident originated from \textbf{a certain international airport}. Who is the
director of the film?
\end{fortrefinementexample}
\caption{Shortcut repair by broadening the aviation reference and relaxing the
runtime constraint.}
\label{fig:adv_case_broadening}
\end{figure}
\FloatBarrier

\begin{figure}[!htbp]
\centering
\begin{fortrefinementexample}{Refinement Case 2: Narrowing Over-fuzzed Clues}
\normalsize
\textbf{Before.}
An individual was born in a certain year and died 61 years after that birth year.
The individual was a member of an institution headquartered in a building in
\textbf{a certain European capital}, built in a year in the 1910s. This
institution was a junior coalition partner to another institution founded in
\textbf{a certain year}. The individual served in a ministerial position
involving economic and communications matters in \textbf{a region} where another
individual also served. The other individual worked as a lawyer until 50 years
after that other institution's founding year. The individual had been threatened
multiple times before death. Who is the individual described above?

\smallskip
\textbf{Refinement.}
The descriptions ``a certain European capital'', ``a certain year'', and
``a region'' are broad enough to admit an unrelated figure. The revision
narrows these geographic, temporal, and administrative descriptions to exclude
the spurious answer while keeping the relevant names implicit.

\smallskip
\textbf{After.}
An individual was born in a certain year and died 61 years after that birth year.
The individual was a member of an institution headquartered in a building in
\textbf{the capital of a federal republic in Central Europe}, built in a year
in the 1910s. This institution was a junior coalition partner to another
institution founded in \textbf{the mid-20th century}. The individual served in
a ministerial position involving economic and communications matters in
\textbf{a federated state} where another individual also served. The other
individual worked as a lawyer until 50 years after that other institution's
founding year. The individual had been threatened multiple times before death.
Who is the individual described above?
\end{fortrefinementexample}
\caption{Ambiguity repair by narrowing over-fuzzed geographic, temporal, and
administrative descriptions.}
\label{fig:adv_case_narrowing}
\end{figure}
\FloatBarrier
\clearpage

\section{Experiment Details}
\subsection{Experimental setup}
\label{app:exp_setup}

\paragraph{Models and training.}We train \textsc{FORT}-Searcher using SFT only on search trajectories
generated from \textsc{FORT} tasks. We use Qwen3.6-35B-A3B~\citep{qwen36_35b_a3b}
as the primary backbone and further evaluate Qwen3.5-9B and
Qwen3.5-4B~\citep{qwen3.5} to examine scaling effects. We use DeepSeek-V4-Pro-0813 as the teacher model.
Detailed training parameters are summarized in Table~\ref{tab:sft_hyperparameters}.
\paragraph{Benchmarks.}
We evaluate search capabilities on BrowseComp~\citep{wei2025browsecomp}
and xbench-DeepSearch~\citep{chen2025xbench} for long-context search,
BrowseComp-ZH~\citep{zhou2025browsecomp2} for Chinese search, and
WideSearch~\citep{wong2026widesearch} for
multi-answer retrieval and Markdown organization.
For robustness, we report the average performance over repeated
evaluations, with one run for BrowseComp, ten runs for xbench-DeepSearch,
and five runs for BrowseComp-ZH and WideSearch. We use Gemini-3.5-Flash as the evaluator.
WideSearch is evaluated using item-level F1.
To avoid unfair comparisons, we blocked domains from Hugging Face in the
testing environment.

\paragraph{Baselines and evaluation protocol.}We compare \textsc{FORT}-Searcher with large-scale search agents and
30B-scale models. The large-scale agents include DeepSeek-V4-Pro-0813 and
DeepSeek-V4-Flash-0731~\citep{deepseekai2026deepseekv4}, GLM-5.3~\citep{zeng2026glm},
Kimi-K3~\citep{team2026kimi}, and Qwen3.8-Max~\citep{qwen38}.
The 30B-scale baselines include REDSearcher~\citep{chu2026redsearcher},
Apodex-1.0-mini~\citep{apodex2026}, BigBang-v1~\citep{bigbang2026},
and Qwen3.6-35B-A3B~\citep{qwen36_35b_a3b}.
We additionally evaluate Qwen3.5-9B and Qwen3.5-4B~\citep{qwen3.5} before and after
fine-tuning. All evaluations follow the official protocols when
available, and otherwise use the same evaluation harness and budget as \textsc{FORT}-Searcher. 

\paragraph{Context management.}
We follow the context-management framework of
MiroThinker~\citep{team2026mirothinker}. Each rollout is allocated a certain
budget of interaction turns, with tool observations retained to
enable evidence reuse. If the model fails to produce a final answer
within the budget, we reset the context and restart from the original
question. Each evaluation instance allows up to five context resets.
\begin{table}[h]
\centering
\caption{Key SFT hyperparameters for
\textsc{FORT}-Searcher.}
\label{tab:sft_hyperparameters}
\normalsize
\setlength{\tabcolsep}{10pt}
\renewcommand{\arraystretch}{1.12}
\begin{tabular}{@{}ll@{}}
\toprule
\textbf{Hyperparameter} & \textbf{Value} \\
\midrule
Fine-tuning method
& Full-parameter SFT (language model) \\
Training epochs
& 4 \\
Maximum sequence length
& 262{,}144 tokens \\
Sequence packing
& Enabled \\
Global batch size
& 32 \\
Micro-batch size
& 1 \\
\midrule
Peak learning rate
& $3 \times 10^{-5}$ \\
Minimum learning rate
& $2 \times 10^{-6}$ \\
Learning-rate schedule
& Cosine decay \\
Warmup steps
& 10 \\
Weight decay
& $0.01$ \\
Adam coefficients $(\beta_1, \beta_2)$
& $(0.9, 0.95)$ \\
Gradient clipping threshold
& $1.0$ \\
\midrule
MoE auxiliary loss coefficient
& $10^{-3}$ \\
MoE expert capacity factor
& $2.0$ \\
\bottomrule
\end{tabular}
\end{table}

\paragraph{Data generation.}
We construct evidence graphs with a maximum depth of $D=4$ and a
node-addition budget of $B=20$. A retrieval-cost threshold of
$\tau_{\mathrm{th}}=30$ retrieval calls is used to identify
shortcut-prone drafts for refinement, but does not determine whether a
draft is retained. Each draft is refined for at most one round.
The final training set contains 40K successful trajectories.
Table~\ref{tab:data_generation_settings} summarizes these settings.

\begin{table}[t]
\centering
\caption{Key data-generation settings for \textsc{FORT}.}
\label{tab:data_generation_settings}
\normalsize
\setlength{\tabcolsep}{6pt}
\renewcommand{\arraystretch}{1.1}
\begin{tabular}{@{}lr@{}}
\toprule
\textbf{Setting} & \textbf{Value} \\
\midrule
Graph depth limit $D$ & 4 \\
Node-addition budget $B$ & 20 \\
Refinement threshold $\tau_{\mathrm{th}}$ & 30 \\
Maximum refinement rounds & 1 \\
Final training-set size & 40K \\
\bottomrule
\end{tabular}
\end{table}
\subsection{Results on Additional Backbones}
\label{app:additional_backbones}

To examine whether the effectiveness of \textsc{FORT} training
data extends beyond Qwen3.6-35B-A3B, we additionally fine-tune
Qwen3-30B-A3B-Thinking-2507 and Qwen3.5-35B-A3B on
\textsc{FORT} trajectories using SFT.
We evaluate these models on BrowseComp and compare them with
the REDSearcher and Apodex-1.0-mini reference results,
respectively. 

Table~\ref{tab:additional_backbones} shows that the two
additional \textsc{FORT}-Searcher variants achieve
77.2 and 79.8 on BrowseComp, exceeding the corresponding
reference scores. These results support the applicability of
\textsc{FORT} data across backbone versions.

\begin{table}[h]
\centering
\caption{BrowseComp results across backbone versions.}
\label{tab:additional_backbones}

\normalsize
\setlength{\tabcolsep}{5pt}
\renewcommand{\arraystretch}{1.12}
\begin{tabular}{@{}llr@{}}
\toprule
\textbf{Model}
& \textbf{Backbone}
& \textbf{BC (\%)} \\
\midrule
REDSearcher
& Qwen3-30B-A3B-Thinking-2507
& 57.4 \\
\textsc{FORT}-Searcher
& Qwen3-30B-A3B-Thinking-2507
& 77.2 \\ 
\midrule
Apodex-1.0-mini
& Qwen3.5-35B-A3B
& 71.5 \\
\textsc{FORT}-Searcher
& Qwen3.5-35B-A3B
&  79.8 \\ 
\midrule
BigBang-v1
& Qwen3.6-35B-A3B
& 76.5 \\
\textsc{FORT}-Searcher
& Qwen3.6-35B-A3B
& 83.3 \\
\bottomrule
\end{tabular}
\end{table}
\subsection{Cross-Solver Robustness of Dataset Difficulty Analysis}
\label{app: Cross-Solver}
To examine whether the observed difficulty pattern depends on the
diagnostic solver, we repeat the dataset comparison using
Qwen3.5-397B-A17B under the same retrieval budget. For each dataset,
we evaluate the same 200 sampled questions, computing accuracy over all 200 questions and trajectory diagnostics over the successful subset of the same runs.

Although absolute values and some baseline orderings change,
\textsc{FORT} again yields the highest solving cost and latest answer
exposure. Its accuracy is comparable to REDSearcher
(60.5\% vs.\ 59.5\%), while $\widehat{\Omega}$ and
$T_{\mathrm{hit}}$ remain substantially higher. This indicates that
the observed increase in pre-answer search is not simply explained by
lower solver accuracy and persists across diagnostic solvers.
\begin{table}[t]
\centering
\caption{Dataset difficulty comparison under trajectory signatures
using Qwen3.5-397B-A17B as the diagnostic solver. For each dataset,
accuracy is computed over 200 randomly sampled questions, while trajectory
diagnostics are computed over the successful subset of the same runs.}
\normalsize
\setlength{\tabcolsep}{4pt}
\renewcommand{\arraystretch}{1.12}
\begin{tabular}{@{}lcccc@{}}
\toprule
\textbf{Source}
& \textbf{Acc. (\%)}
& $\boldsymbol{\widehat{\Omega}}$
& $\boldsymbol{T_{\mathrm{hit}}}$
& $\boldsymbol{\widehat{p}_{\mathrm{prior}}(\%)}$ \\
\midrule

InfoSeek~\citep{xia2025open}
& 97.5
& 20.6
& 5.7
& 2.1 \\

MiroVerse-Voyager~\citep{team2025mirothinker}
& 99.0
& 30.6
& 5.9
& 5.6 \\

DeepDive~\citep{lu2025deepdive}
& 94.0
& 47.7
& 15.5
& 7.4 \\

DeepResearch-9K~\citep{pham2025sealqa}
& 49.5
& 47.8
& 3.4
& 27.3 \\

OpenSeeker~\citep{du2026openseeker}
& 67.5
& 84.7
& 9.3
& 31.9 \\

REDSearcher~\citep{chu2026redsearcher}
& 59.5
& 92.1
& 18.7
& 11.8 \\

\midrule

\textbf{\textsc{FORT}}
& 60.5
& \textbf{141.0}
& \textbf{46.9}
& 10.7 \\

\bottomrule
\end{tabular}
\label{tab:cross_solver}
\end{table}
\subsection{Trajectory-Level Proxies and Diagnostic Validation}
\subsubsection{Trajectory-Level Proxies for Difficulty Factors}
\label{app:data_comp}

To examine how the difficulty factors in Appendix~\ref{app:formal} are reflected in
realized search behavior, we conduct a trajectory-level diagnosis on successful
question--trajectory pairs. The exact values of the theoretical quantities in
Appendix~\ref{app:formal} are not feasible to compute in open-domain web search. In particular, exact evaluation of $s(\mathcal P)$ requires
determining the corresponding answer set, while
$M_{\mathrm{ev}}(\mathcal P)$ and $\mathrm{dep}(\mathcal P)$ require
optimization over evidence-acquisition routes. The solver-side
cost difference $U_{\pi_0}(q)$ additionally depends on the reference
cost $D_{\mathrm{post}}(q)$, which is not directly observed in these
experiments. This difference does not isolate the effect of prior
knowledge.

We therefore do not treat the following measurements as direct estimates of the
theoretical quantities. Instead, we use them as trajectory-level proxies for the
realized effects of the difficulty factors analyzed in our framework.
Specifically, we randomly sample 200 successful question--trajectory pairs from
open-source deep-search data and 200 from \textsc{FORT}. We use GPT-5.5 to
inspect each complete question and trajectory, decompose the question into
atomic clues, and align each clue with the corresponding search steps and
retrieved evidence. Based on this annotation, we define four observable proxies.

For \emph{subset selectivity}, we examine whether the search evidence associated
with each individual clue already narrows the candidate space to a very small
set. We mark a clue as low-width if the clue-specific search evidence yields
only one or two plausible candidate answers. We report the percentage of such
clues as \(R_{\mathrm{low}}\). A lower \(R_{\mathrm{low}}\) indicates that
individual clues are less likely to identify the answer on their own, making
small identifying subsets less available in realized search.

For \emph{evidence dispersion}, we assign each verified clue to one primary
evidence source that supports it. If the same web page or evidence source
verifies multiple clues, it is counted only once. We then compute the normalized
evidence dispersion:
\begin{equation}
R_{\mathrm{ev}}=\widehat{M}_{\mathrm{ev}}/n_c \times 100\%,
\end{equation}
where \(\widehat{M}_{\mathrm{ev}}\) is the number of distinct primary evidence
sources and \(n_c\) is the number of clues in the question. A higher
\(R_{\mathrm{ev}}\) indicates more dispersed evidence, while a lower value
suggests that multiple clues are co-covered by a small number of sources.

For \emph{dependency depth}, we do not estimate the oracle quantity
\(\mathrm{dep}(\mathcal{P})\). Instead, we measure the retrieval cost of the
most costly realized dependency chain, denoted by
\(\widehat{C}_{\mathrm{dep}}\). We first keep only clue-relevant search steps
before the first answer hit. We then check whether each step uses an entity
name, attribute, date, title, alias, or intermediate fact first exposed by
earlier retrieved evidence. If a later search depends on information exposed by
an earlier step, the two steps are considered part of the same realized
dependency chain. Among all realized dependency chains, we select the one with
the highest retrieval cost and count the actual retrieval steps needed to
complete it. A higher \(\widehat{C}_{\mathrm{dep}}\) indicates stronger serial
dependence and greater retrieval effort before the answer is exposed.

To examine whether delayed answer exposure reflects stronger pre-answer evidence dependencies, we further compute the Spearman
correlation between $T_{\mathrm{hit}}$ and
$\widehat{C}_{\mathrm{dep}}$ over the annotated trajectories.
We report the correlation separately for open-source data and FORT.
\begin{table}[h]
\centering
\caption{Correlation between answer hit time and realized pre-answer
dependency cost.}
\label{tab:thit_dep_corr}
\normalsize
\setlength{\tabcolsep}{8pt}
\renewcommand{\arraystretch}{1.1}
\begin{tabular}{@{}lc@{}}
\toprule
\textbf{Data}
& $\boldsymbol{\rho(T_{\mathrm{hit}},
\widehat{C}_{\mathrm{dep}})}$ \\
\midrule
Open-source & 0.31 \\
FORT & 0.36 \\
\bottomrule
\end{tabular}
\end{table}

The positive correlations show that later answer exposure is associated
with more costly realized pre-answer dependency chains. This supports
$T_{\mathrm{hit}}$ as an observable diagnostic of answer-exposure delay,
rather than a direct estimate of the minimum identifying-route cost
$Q_\Sigma^\star$.

For answer-before-evidence behavior, we use the indicator defined in Section~\ref{sec:diagnostics}. We report the fraction of
trajectories where GPT-5.5 judges, based on semantic equivalence rather than
string matching, that the model proposes the gold answer before retrieved
evidence exposes it, denoted by \(\widehat{p}_{\mathrm{prior}}\).

\begin{table}[t]
\centering
\normalsize
\caption{
Trajectory-level proxies for realized difficulty factors on 200 successful
question--trajectory pairs from each data source. \(\widehat{\Omega}\) denotes
the empirical average solving cost. \(R_{\mathrm{low}}\) denotes the percentage
of clues whose clue-specific search evidence yields only one or two plausible
candidate answers. \(R_{\mathrm{ev}}\) denotes normalized evidence dispersion.
\(\widehat{C}_{\mathrm{dep}}\) denotes the retrieval cost along the most costly
realized pre-answer dependency chain. \(\widehat{p}_{\mathrm{prior}}\) denotes
the fraction of trajectories where the model proposes the gold answer before
retrieved evidence exposes it.
}
\renewcommand{\arraystretch}{1.15}
\setlength{\tabcolsep}{3.5pt}
\begin{tabular}{@{}lccccc@{}}
\toprule
\textbf{Data}
&
\makecell[c]{\(\widehat{\Omega}\)\\\(\uparrow\)}
&
\makecell[c]{\(R_{\mathrm{low}}\)\\\((\%)\downarrow\)}
&
\makecell[c]{\(R_{\mathrm{ev}}\)\\\((\%)\uparrow\)}
&
\makecell[c]{\(\widehat{C}_{\mathrm{dep}}\)\\\(\uparrow\)}
&
\makecell[c]{\(\widehat{p}_{\mathrm{prior}}\)\\\((\%)\downarrow\)} \\
\midrule
Open-source
& 51.5 & 58.2 & 76.3 & 2.7 & 17.0 \\
\textsc{FORT}
& \textbf{100.3} & \textbf{43.0} & \textbf{88.2} & \textbf{5.5} & \textbf{13.0} \\
\bottomrule
\end{tabular}
\label{tab:trajectory_factor_proxies}
\end{table}

Table~\ref{tab:trajectory_factor_proxies} shows that \textsc{FORT} moves all
trajectory-level proxies in the desired direction: compared with open-source
deep-search data, it induces higher retrieval effort, less over-selective
individual clues, more dispersed supporting evidence, more costly pre-answer
dependency chains, and a lower answer-before-evidence rate. These proxies provide an operational connection between the theoretical
difficulty factors and realized search behavior. They do not estimate the
oracle minimum route cost in Appendix~A, but show that \textsc{FORT}
exhibits lower single-clue selectivity, greater evidence dispersion,
more costly realized pre-answer dependency chains, and less
answer-before-evidence behavior under the evaluated search setting.

\subsubsection{Human Validation of the Answer-Before-Evidence Indicator}
\label{app:prior_validation}

To validate the answer-before-evidence indicator, we manually inspect 200
trajectories, including 100 with $I_{\mathrm{prior}}=1$ and 100 with
$I_{\mathrm{prior}}=0$. Three qualified annotators with graduate-level
research training independently inspect each question and its
corresponding search trajectory. The trajectories are randomly shuffled,
and the annotators are not informed of the indicator value.

For each trajectory, annotators determine whether the model proposes
the gold answer before retrieval provides supporting
evidence. Semantically equivalent aliases are treated as the same
answer. A case is not confirmed when the answer has already been
supported by retrieved evidence before the model proposes it, or when
the automatic answer match is incorrect.

\begin{table}[h]
\centering
\caption{Human validation of the answer-before-evidence indicator.}
\label{tab:prior_validation}
\normalsize
\setlength{\tabcolsep}{8pt}
\renewcommand{\arraystretch}{1.1}
\begin{tabular}{@{}lcc@{}}
\toprule
\textbf{Indicator} & \textbf{ Trajectories} & \textbf{Human-confirmed (\%)} \\
\midrule
$I_{\mathrm{prior}}=1$ & 100 & 89.0 \\
$I_{\mathrm{prior}}=0$ & 100 & 3.0 \\
\bottomrule
\end{tabular}
\end{table}

Human annotation confirms answer-before-evidence behavior in 89.0\% of
trajectories with $I_{\mathrm{prior}}=1$, compared with 3.0\% of those
with $I_{\mathrm{prior}}=0$. These results support $\widehat{p}_{\mathrm{prior}}$ as a behavioral
proxy for answer-before-evidence shortcuts rather than a direct measure
of the model's internal prior knowledge.
\subsection{Question Validity and Uniqueness Audit}
\label{app:validity_audit}

To examine whether increased search difficulty comes at the cost
of question validity, we conduct a manual audit covering three
groups. We randomly sample 50 questions accepted without
refinement, 50 successfully refined shortcut-prone cases, and
50 successfully refined initially unsolved cases. For each
successfully refined case, we inspect both the original
pre-refinement draft and the corresponding post-refinement
question, enabling paired comparisons before and after refinement.

Three qualified annotators with graduate-level research training
independently inspect each question together with its gold answer
and supporting evidence. Questions are randomly shuffled, and the
annotators are not informed of the refinement category or whether
a question is the pre- or post-refinement version. Each annotator
evaluates four properties: answer correctness, source support,
answer uniqueness, and ambiguity.

Answer correctness indicates whether the provided gold answer is
factually correct for the complete question. Source support requires
the provided evidence to support the constraints needed to establish
the gold answer. Answer uniqueness requires that no plausible
alternative answer satisfies the complete set of constraints.
A question is marked unambiguous when its wording admits no
materially different interpretation that leads to an alternative
answer.






\begin{table}[t]
\centering
\caption{Manual validity audit of \textsc{FORT} questions.
Values are percentages.}
\label{tab:validity_audit}
\normalsize
\setlength{\tabcolsep}{5.5pt}
\renewcommand{\arraystretch}{1.12}

\begin{tabular}{@{}llcccc@{}}
\toprule
\textbf{Group}
& \textbf{Version}
& \textbf{Correct}
& \textbf{Supported}
& \textbf{Unique}
& \textbf{Unambiguous} \\
\midrule

Accepted
& Final
& 96.0 & 96.0 & 94.0 & 92.0 \\

\midrule
Shortcut-prone
& Before
& 98.0 & 96.0 & 96.0 & 98.0 \\
& After
& 94.0 & 94.0 & 92.0 & 92.0 \\

\midrule
Initially unsolved
& Before
& 88.0 & 86.0 & 72.0 & 74.0 \\
& After
& 94.0 & 92.0 & 90.0 & 88.0 \\

\bottomrule
\end{tabular}
\label{tab:man}
\vspace{2pt}
\end{table}

As shown in Table~\ref{tab:man}, questions accepted without refinement show consistently
high validity across all four metrics. For shortcut-prone cases, validity remains high after refinement,
with only modest decreases in correctness, uniqueness, and clarity.
For initially unsolved cases, refinement substantially improves
uniqueness from 72.0\% to 90.0\% and unambiguous formulation from
74.0\% to 88.0\%, while preserving answer correctness and source
support. Overall, the audit shows that the two refinement modes
serve distinct purposes without substantially compromising task validity.
\subsection{Additional Ablation Analyses}
\label{app:downstream_ablation}
\paragraph{Downstream evaluation.}
We further examine whether the cumulative configurations in
Figure~\ref{fig:ab_short} differ in downstream training value.
We select the full configuration, the configuration after removing
Cycle, Long-tail, and Derived, and the final configuration after all
six controls have been removed. For each setting, we use corresponding
DeepSeek-V4-Pro-0813 trajectories and fine-tune Qwen3.5-4B with the same
training recipe and inference setting.

\begin{table}[h]
\centering
\caption{Downstream performance of representative configurations
from the cumulative shortcut-resistance ablation.}
\label{tab:downstream_shortcut_ablation}
\normalsize
\setlength{\tabcolsep}{7pt}
\renewcommand{\arraystretch}{1.1}
\begin{tabular}{@{}lc@{}}
\toprule
\textbf{Configuration}
& \textbf{BrowseComp} \\
\midrule
Full
& 61.9 \\
$-$ Cycle $-$ Long-tail $-$ Derived
& 58.1 \\
$-$ Cycle $-$ Long-tail $-$ Derived $-$ Diversity
$-$ Generic $-$ Fuzzing
& 54.9 \\
\bottomrule
\end{tabular}
\end{table}

As shown in Table~\ref{tab:downstream_shortcut_ablation}, BrowseComp decreases from 61.9 for the full configuration to
58.1 after removing Cycle, Long-tail, and Derived, and further to
54.9 after all six controls are removed. Together with
Figure~\ref{fig:ab_short}, this shows that progressively removing
shortcut-resistant controls is accompanied by easier realized search
behavior and lower downstream training value.
\paragraph{Isolated component ablations.}
To complement the cumulative analysis in Figure~\ref{fig:ab_short},
we independently remove long-tail selection, derived-fact construction,
and fuzzing from the full \textsc{FORT} pipeline. Each configuration
removes only the specified component while keeping the remaining
construction and evaluation settings fixed. We evaluate 3000 questions
per configuration, computing accuracy over all questions and trajectory
diagnostics over successful runs.

\begin{table}[t]
\centering
\caption{Isolated ablations of representative shortcut-resistant
components.}
\label{tab:isolated_ablation}
\normalsize
\setlength{\tabcolsep}{4pt}
\renewcommand{\arraystretch}{1.1}
\begin{tabular}{@{}lcccc@{}}
\toprule
\textbf{Configuration}
& \textbf{Acc. (\%)}
& $\boldsymbol{\widehat{\Omega}}$
& $\boldsymbol{T_{\mathrm{hit}}}$
& $\boldsymbol{\widehat{p}_{\mathrm{prior}}(\%)}$ \\
\midrule
Full
& 64.2 & 92.3 & 16.1 & 12.7 \\
$-$ Long-tail
& 72.0 & 76.4 & 13.0 & 15.3 \\
$-$ Derived
& 67.3 & 78.6 & 14.2 & 13.7 \\
$-$ Fuzzing
& 78.8 & 72.7 & 9.3 & 14.8 \\
\bottomrule
\end{tabular}
\end{table}

As shown in Table~\ref{tab:isolated_ablation}, removing long-tail
selection increases the answer-before-evidence rate, while removing
derived facts substantially reduces retrieval cost and removing fuzzing
causes the largest drop in answer-hit time. Across all three isolated ablations, accuracy increases and the trajectory
diagnostics shift toward easier search, with component-specific changes
consistent with the intended shortcut risks.

\end{document}